\documentclass[journal]{IEEEtran}

\ifCLASSINFOpdf
\else
 \fi

\usepackage{array}
\usepackage{hyperref}
\usepackage{lineno}
\modulolinenumbers[5]
\usepackage{amsmath,amssymb,amsfonts}
\usepackage{amsthm}

\usepackage{multirow}
\usepackage{graphicx}
\usepackage[ruled,linesnumbered]{algorithm2e}
\let\oldnl\nl
\newcommand{\nonl}{\renewcommand{\nl}{\let\nl\oldnl}}

\newtheorem{lemma}{Lemma}
\newtheorem{theorem}{Theorem}
\newtheorem{corollary}{Corollary}
\newtheorem{definition}{Definition}
\usepackage{hyperref}
\usepackage[noadjust]{cite}
\usepackage{enumitem}
\usepackage[font={footnotesize}]{caption}
\usepackage[font={footnotesize}]{subcaption}

\usepackage{pifont}

\usepackage{xcolor}

\usepackage{fancyhdr}

\fancypagestyle{firststyle}{
   \fancyhf{}
   \fancyfoot[C]{\footnotesize \thepage}
   \fancyhead[C]{\footnotesize Accepted for publication in IEEE Transactions on Emerging Topics in Computing (TETC)}
   
}
\begin{document}

\title{Bit-ViP: Leveraging Bit-planes to Preserve Visual Privacy in Images through Obfuscation}
\author{
 	\IEEEauthorblockN{Anonymous}}

\author{
	\IEEEauthorblockN{Vishesh Kumar Tanwar\IEEEauthorrefmark{1}, Ashish Gupta\IEEEauthorrefmark{2}, Sanjay Madria\IEEEauthorrefmark{1}, and Sajal K. Das\IEEEauthorrefmark{1}}
 
	 \IEEEauthorblockA{
  \IEEEauthorrefmark{1}Department of Computer Science, Missouri University of Science and Technology, USA \\ \IEEEauthorrefmark{2} Department of Computer Science and Engineering, BITS Pilani Dubai Campus, Dubai, UAE\\
 $\left\{\text{vishesh.tanwar, madrias, sdas}\right\}$@mst.edu, ashish@dubai.bits-pilani.ac.in\\
}}

\maketitle
\thispagestyle{firststyle}

\begin{abstract}
The unprecedented growth of computer vision applications, such as surveillance systems and social media, raises security and visual privacy concerns, especially when data is stored on cloud servers. Image obfuscation offers a way to preserve visual privacy while maintaining an adequate level of usability; thus, it has been a topic of great interest in recent years. However, prior obfuscation schemes are either vulnerable to malicious attacks, such as model inversion to reconstruct original images from obfuscated images, or generate non-trainable obfuscated images, making them unusable for achieving reasonable accuracy. This paper proposes a novel bit-plane-based image obfuscation scheme, {\em Bit-ViP}, to preserve visual privacy for image-based recognition tasks. The Bit-ViP scheme produces secure, usable images by incorporating an innovative end-to-end obfuscation function. While doing so, the obfuscated image would contain non-invertible noise (generated by Lorenz's chaotic system and differential privacy), making it hard for an adversary to reconstruct the original image. We conduct extensive experiments on two popular activity recognition datasets, namely UCF101 and HMDB51, to validate the effectiveness of Bit-ViP. In the face of attacks on reconstruction, pixel frequency, information entropy, and pixel inter-correlation, we present a rigorous security analysis demonstrating tangible improvements over existing schemes.              
\end{abstract}
\IEEEpeerreviewmaketitle

\section{Introduction}\label{intro}
Due to advancements in digital camera technology and an unprecedented surge in the use of surveillance systems~\cite{doula2022vr}, the amount of image and video data has exploded on cloud servers. Traditionally, such data is uploaded in {\em plain form} (i.e., unobfuscated) to leverage the cloud's resources for computer vision tasks such as object detection and face identification, which raises security and privacy concerns for the users because the involved companies may exploit the private visual information~\cite{rajpoot2014security} or the cloud could be compromised. According to a survey~\footnote{https://www.usatoday.com/story/tech/2020/01/28/not-reading-the-small-print-is-privacy-policy-fail/4565274002/} conducted over 2000 users, 97\% of people accept in-place legal terms and conditions without reading carefully, further increasing the severity level of the issue. Consider a scenario of airport surveillance security in which multiple CCTV cameras are deployed to detect suspicious human activities. As shown in Fig.~\ref{obfuscation_schemes}, the recorded data (image and video) are continuously uploaded to the cloud, which may invite adversaries to obtain passengers' visual personal identification information (V-PII), such as their face, gender, race, travel details{\em, etc.}, by attacking the server. In such scenarios, a unified {\em security and privacy assurance} scheme must be integrated with cloud services to create secure data storage without harming the usability of cloud-based applications. In addition, privacy-preserving schemes protect the visual privacy of image information, especially in scenarios where raw data transmission to the cloud for processing may pose security and privacy risks, thereby enhancing user trust in sharing their data for machine learning (ML) purposes. This is especially important in applications such as healthcare and surveillance, where data security and privacy are critical. However, most existing schemes either focus on data security~\cite{kumar2021secured} or on privacy-preserving processing~\cite{myneni2022scvs}, but not both simultaneously, which we address in this paper.

\begin{figure}[!t]
\centering
\includegraphics[width=1\linewidth,height=6.5cm]{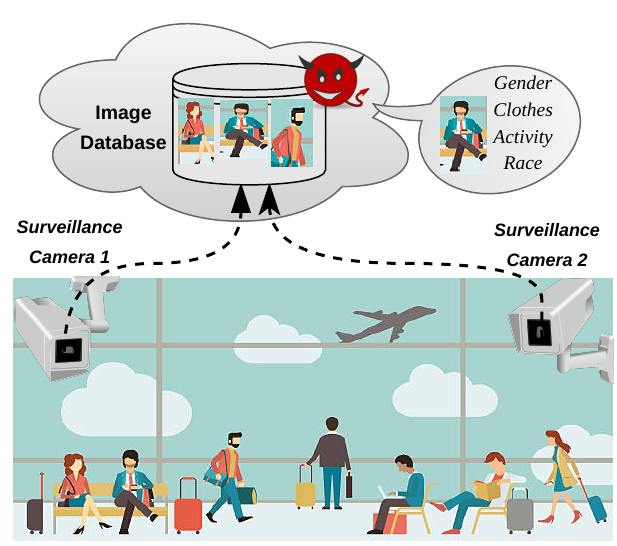}
\vspace{-0.5cm}
\caption{Illustrating an airport surveillance scenario. Images uploaded to a cloud server may expose visual attributes such as gender, race, and clothes.}
\label{obfuscation_schemes}
\vspace{-0.6cm}
\end{figure}

A straightforward way to preserve individuals' visual identity is to obfuscate their faces via {\em blurring}~\cite{ilia2015face} before storing the data on a cloud. Extending the obfuscation from the face to the entire image can be accomplished by {\em down-sampling}~\cite{kim2019privacy,li2021deepblur} to conceal V-PII that otherwise is apparent in the high-resolution images. On a similar track, the authors in~\cite{rajput2020privacy} incorporated downsampling and added Gaussian noise to the underlying image to obtain an obfuscated image that conceals activity information. {\em Scrambling}~\cite{jeevitha2021novel} and image encryption~\cite{zhou2022new} are data security methods that break the inter-correlation of the pixels, thereby making the resultant image visually indecipherable and unusable to ML applications. Slightly different from obfuscation, cryptography-based approaches for secure data processing, such as fully homomorphic encryption (FHE)~\cite{gilad2016cryptonets}, secure multi-party computations~\cite{knott2021crypten}, and Garbled Circuits~\cite{saleem2018recent} offer strong visual privacy; however, these approaches are computationally and communication-ineffective with ML algorithms, imparting three major limitations:

\text{(1)} {\em Security and usability (accuracy):} Though the prior encryption schemes~\cite{gilad2016cryptonets,jeevitha2021novel,chamikara2020privacy} offer only good security by minimizing or breaking the pixels' inter-correlation, the generated obfuscated images are not usable enough for training deep neural networks (DNNs) to achieve adequate classification accuracy (in other words, usability of the encrypted data). On the other hand, blurred~\cite{ilia2015face} or downsampled~\cite {li2021deepblur} images preserve essential information for training DNNs, thereby demonstrating the high usability of obfuscated data, but they are vulnerable to malicious attacks such as data reconstruction~\cite{NEURIPS2021_fa84632d}. Thus, the underlying privacy-preserving schemes lack essential security measures and may lead to information leakage from images. 

\text{(2)} {\em Trusted server:} In current client-server settings, the clients are bound to assume that the server is trustworthy because the underlying obfuscation schemes~\cite{kim2019privacy,li2021deepblur} are designed and provided by the server. In most cases, the schemes are developed using data-driven DNNs whose inverse can be obtained by reverse training, posing a risk of adversarial attacks~\cite{NEURIPS2021_fa84632d} if the server has malicious intent. Thus, users may be reluctant to adopt such solutions for sensitive applications.  

\text{(3)} {\em Model architecture dependency:} As the existing schemes~\cite{kim2019privacy,li2021deepblur} employ fixed DNNs to generate an obfuscated image, they restrict the server to using only a fixed model architecture, which may not be suitable for the user's application interests, thereby limiting users' flexibility.\\\vspace{-0.3cm}

\noindent \textbf{Contributions:}
Addressing the above limitations, we ensure the security and privacy of images stored on the cloud platform to enable efficient use, such as training a DNN for classification. Our major contributions are:
\begin{itemize}
\item We propose a novel end-to-end bit-planes-based image obfuscation scheme, {\em Bit-ViP}, to preserve V-PII for creating secure image storage. With an innovative design, we incorporate security into the original images by partitioning them into non-overlapping rectangular blocks, thereby introducing chaotic confusion at the bit-plane level. Then, non-invertible, bounded, and differentially private noise is added to each block to improve usability. 
\item Unlike prior approaches, the Bit-ViP offers more robust security against malicious attacks like reconstruction and de-identification attacks
by injecting non-invertible noise at a bit-plane level while making a comparatively lesser compromise on the recognition accuracy. 
\item With lightweight computations, the Bit-ViP is suitable for low-configuration devices; thus, the obfuscation function need not be known to the server, making it independent of the server's intentions and robust against adversarial access.
\item By employing two benchmark activity recognition datasets, UCF101~\cite{soomro2012ucf101} and HMDB51~\cite{kuehne2011hmdb}, we evaluate the usability (training the DNNs) and security effectiveness of the proposed method through rigorous qualitative and quantitative security analysis, demonstrating its superiority over existing recognition and data security schemes, respectively. 
\end{itemize}

This paper extends our prior work~\cite{tanwar2023preserving}, which introduces an end-to-end image obfuscation scheme that secures data against reconstruction attacks and is supported by rigorous theoretical analysis, thereby solidifying the Bit-ViP framework. The proposed obfuscation methodology has been advanced by integrating bit-plane perturbations with an exponential differential privacy (DP) mechanism, thereby achieving greater security. Extensive new experiments demonstrate the security and effectiveness of our novel approach, Bit-ViP, through a comparative analysis with well-established algorithms, highlighting the superiority of Bit-ViP. We also address real-time scalability by evaluating the computation time of the Bit-ViP framework. Additionally, this scheme explains critical concepts, such as \textit{Privacy} and \textit{Utility} for the considered scenario, normalization, and binary thresholding while removing ambiguities related to labels. This extended version introduces a more refined problem statement with two practical use cases.

\vspace{0.1cm}

\noindent \textbf{Paper organization:} Section~\ref{related_work} discusses the recent related research on obfuscation schemes, followed by preliminaries in Section~\ref{prelims}. The problem statement and threat model are explained in Section~\ref{ps_tm}, and Section~\ref{proposed_method} proposes our image obfuscation method, with theoretical and qualitative analyses. We conduct extensive activity recognition experiments for the usability of obfuscated data and security attacks, and the results are reported in Section~\ref{experi_results}. Finally, Section~\ref{conclusion} concludes the paper with promising future directions.

\vspace{-0.1in}
\section{Related Work}\label{related_work}
This section presents the current status of the literature on obfuscation methods. We divided the prior works into two categories based on the underlying process.

\vspace{-0.1in}
\subsection{Learning-based obfuscation methods}\label{learning_based}
By introducing a privacy-preserving system for accessing users' images stored in the social-media cloud, the authors in~\cite{ilia2015face} focused on blurring users' faces and preventing identity leakage. In~\cite{kim2019privacy}, the face resolution is reduced to an extremely low level (via downsampling) to preserve privacy, while the image background is enhanced to improve the models' learning capability. On the same track, {\em DeepBlur}~\cite{li2021deepblur} is proposed to prevent face re-identification attacks by obfuscating the latent feature space of the unconditional generative adversarial network~\cite{wang2018transferring}. However, this scheme significantly fails to protect users' other privacy attributes like gender, race, location, clothes{\em, etc.} To improve the classification in extremely Low Resolution (eLR) space, a mosaicing approach is leveraged in~\cite{chou2018privacy}, but still, the resultant obfuscated image is highly vulnerable to differential attacks. Although~\cite {wu2020privacy} introduces an effective optimization approach to improve the underlying DNN's recognition accuracy while achieving strong privacy protection against adversarial attacks, it assumes the presence of a \textit{trusted server}, which we relax in our obfuscation scheme.   
 
Unlike down-sampling-based approaches, the authors in~\cite{bai2019extreme} combined high-resolution and eLR videos to improve activity recognition by leveraging spatial-temporal attention. The spatial-temporal information is also exploited in~\cite{purwanto2019extreme}, where the authors argued that training the model on eLR images and using a teacher-student knowledge distillation approach could enhance visual privacy and improve accuracy. Recently,~\cite{kim2022privacy} brought in the concept of event-to-image and an event-based camera that captures only a fraction of visual information, thereby hiding susceptible details and offering strong visual privacy, contrasting {\em Bit-ViP}s' motivation.

\vspace{-0.1in}
\subsection{Non-learning based obfuscation methods}\label{non_learning_based}
To protect the privacy of social media images, in~\cite{fan2018image}, the images are pixelized by partitioning the input image into blocks and assigning an average intensity value to each block. Later, differentially private Gaussian noise is added to these values to protect V-PII. However, this scheme significantly harms image usability. Similarly, the authors in~\cite{rajput2020privacy} utilized position-based superpixel transformation and {{Gaussian}} noise on RGB-depth video data to build an activity recognition system. In~\cite{jeevitha2021novel}, a block-based image scrambling technique is introduced where discrete wavelet transformation is employed to secure the image. 

The work~\cite{gilad2016cryptonets} presented a scheme, called \textit{CryptoNets}, to obfuscate gray-scale images using the fully homomorphic encryption (FHE) scheme; however, due to computational and storage overheads, this scheme cannot be scaled for complex image and video datasets. To prevent biometric features while authenticating the individual, 
the study in~\cite{chamikara2020privacy} proposed a privacy-preserving face recognition protocol. \cite{chen2020privacy} presented a secure multi-classification scheme to address the privacy leakage in robot systems using DNNs. Two activation-cost function pairs using homomorphic encryption, namely softmax plus log-likelihood and sigmoid plus cross-entropy, were adopted to enable secure computation. Recently, an image obfuscation method based on combining various chaotic systems and hash functions was presented in~\cite{zhou2022new}. Although this method exhibits strong cryptographic properties, it destroys image usability for DNN training. 

\textbf{Limitations:} The methods in Section~\ref{learning_based} leverage DNNs to obtain obfuscated data but are prone to various adversarial attacks by reversing the training. In contrast, the latter ones in Section~\ref{non_learning_based} offer strong cryptographic security against various malicious attacks. However, they fail to preserve image usability for DNNs because they lose the pixels' intercorrelations. Therefore, this work proposes a {\em Bit-ViP} scheme with adequate security while maintaining strong data usability for activity recognition through bit-plane-based obfuscation.\vspace{-0.05in}

\section{Preliminaries}\label{prelims}
To better understand our obfuscation scheme, we first define privacy and usability in this context, then discuss Lorenz's chaotic structures and differential privacy.

\begin{definition}
\textbf{Privacy:} It refers to how obfuscated image data effectively conceals sensitive visual information, ensuring that an adversary cannot reconstruct or identify the original content. Thus, the scheme protects individuals' visual identities, prevents misuse, and maintains the data's utility for intended applications.
\end{definition}
\begin{definition}
\textbf{Usability:} It is defined as the ability of obfuscated image data to maintain sufficient quality and integrity to effectively train DNNs using temporal features, resulting in good recognition accuracy while ensuring security from adversarial attacks.   
\end{definition}

\vspace{-0.25in}
\subsection{Random noise generation using Lorenz's chaotic structures}\label{lorenz}
As {\em Bit-ViP} aims to insert a non-invertible noise into the input image, we reviewed the related literature and found that chaotic structures~\cite{badr2021cancellable} are a favorable choice for securing image information, as they are unpredictable and non-reproducible. Following the work~\cite{badr2021cancellable}, we utilize Lorenz's chaotic structures on three variables $x,y,z$, defined as:
\begin{small}
\begin{align}\label{chaotic_system_eqs}
        & \frac{dx}{dt} = \alpha (y-x), 
        & \frac{dy}{dt} = \beta x - y - xz, \text{ and }        & \frac{dz}{dt} = xy-\gamma z,
\end{align}
\end{small}
\noindent where $\alpha, \beta, \gamma$ are the system parameters. Solution hyperplane of this equation system has chaotic nature for $\left(\alpha=10, \beta=28, \gamma=8/3\right)$; the chosen values are borrowed from~\cite{badr2021cancellable}. Let $\mathcal{S} \in \mathbb{R}^{num_{sol} \times 3}$ denote a solution matrix of Eq.~\ref{chaotic_system_eqs}, where $num_{sol}$ is total number of solutions. To initialize the noise for image obfuscation, we define a non-linear transformation $\zeta: \mathcal{S} \rightarrow M_2 (\mathbb{R})$ as
\vspace{-0.2cm}
\begin{equation}\label{transformation}
    \zeta(a, b, c) =  \begin{bmatrix}-6b+2c & a-b+c+1 \\-a+b-6c-8 & 3a-b+4c \end{bmatrix},
\end{equation}
where ${M}_2 (\mathbb{R})$ is the ring of matrices of order two with real entries. We randomly choose one solution in $\mathcal{S}$, say $s \in \mathbb{R}^3$, and transform as $\zeta(s)$ using Eq.~\ref{transformation}, which generates two vectors $[n_1^1, n_1^2]$ and $[n_2^1, n_2^2]$, taken row-wise. {\em Bit-ViP} leverages these vectors to perturb the QR components in Section~\ref{bit_plane_obfuscation}.
\vspace{-0.3cm}
\subsection{Differential Privacy (DP)} The work~\cite{dwork2014algorithmic} ensures data privacy, where the data remains useful for analytics, specifically for the training of DNNs, without revealing any private information. The fundamental idea is to introduce randomness into the original data so that private information cannot be inferred from the altered data, while preserving statistical information, even with the adversary having unlimited computational resources.
\begin{definition}\label{dp_def}
A randomized mechanism $\mathcal{M}$ is said to be $\epsilon$-Differential Privacy if for any two datasets $\mathcal{D}$ and $\mathcal{D'}$ differs with one element, and all subsets $U$ of possible range $\mathcal{Y}$, 
\begin{equation}\nonumber
 Pr\left[\mathcal{M}\left(\mathcal{D}\right)\in U \right] \leq exp(\epsilon) \cdot Pr\left[\mathcal{M}\left(\mathcal{D'}\right)\in U \right]
\end{equation}
\end{definition}

\section{Problem Statement and Threat Model}\label{ps_tm}
\subsection{Problem Statement}
This paper considers the standard setting for privacy-preserving centralized training of a DNN on a cloud server. In {\em Bit-ViP}, a user has a sensitive image database and would like to train a DNN model for predefined applications on a cloud platform. However, the cloud service provider (CSP) may be ``\textit{untrusted}" and tries to (i) extract users' sensitive information from the transmitted images and (ii) leverage the database to develop their models for other applications, benefiting themselves. However, users' images may contain sensitive and confidential information that could be leaked, significantly impacting users' lives. Similarly, privacy concerns would arise if an adversary compromised the cloud platform. For more clarification, we provide two scenarios.

\noindent\textbf{Scenario 1 (Home and Office Surveillance):} A user with a network of smart surveillance cameras at home and office seeks to develop a personalized activity recognition and object detection DNN model for sensitive labeled videos. Lacking deep learning expertise and resources, the user opts for cloud-based development. However, transmitting unobfuscated (or plain) images poses privacy risks, as they contain sensitive details (visual identity, daily routines, family information, etc.). The proposed solution allows users to obfuscate images before transmission while retaining labels for effective DNN training, ensuring privacy and usability.\\
\noindent\textbf{Scenario 2 (CSP Leveraging User Data):} CSPs often face challenges in collecting robust user data for DNN model development. A CSP might be tempted to utilize user-transmitted labeled surveillance data without consent, raising ethical and legal concerns. Our approach enables users to obfuscate their surveillance videos, preserving visual privacy while the labels remain intact. This ensures the user's visual privacy is not compromised while CSPs use data for model training.

Adding the above scenarios, our proposed scheme can be enhanced for the following applications:
\begin{enumerate}[leftmargin=12pt]
    \item \textbf{Label Retention for Effective Training:} In our approach, the associated labels to each image can be retained and transmitted securely while the images undergo obfuscation. This ensures that the DNNs can still be trained effectively using the labeled, obfuscated images.
    \item \textbf{Semi-Supervised and Unsupervised Learning:} {\em Bit-ViP} opens avenues for semi-supervised and unsupervised learning models, where labels may not be available. In such scenarios, DNNs can learn from obfuscated images, with labels playing a supportive rather than a central role.
\end{enumerate}

\subsection{Threat Model}
The {\em Bit-ViP} involves two entities: an ``\textit{honest}" user $ \mathcal{U} $ who outsources the image database $\mathcal{D}$ to train an activity recognition model while protecting the images' visual information, and an ``\textit{honest-but-curious}\footnote{It always follows the defined protocol but would try to extract information from users' transmitted database.}" CSP, denoted by $\mathcal{C}$, provides the ML model training services with storage and computational resources in a pay-per-use business model. The visual information in each image of $\mathcal{D}$ is obfuscated at the user-end using our proposed image obfuscation scheme in Section~\ref{proposed_method}, resulting in the obfuscated image database, denoted by $\mathcal{D}^{obs}$, which is then transmitted, through a secure communication channel, to $ \mathcal{C}$ for training a DNN model. In our model, $ \mathcal{U}$ does not hold control over the DNN architecture selection and computations performed by $\mathcal{C}$ over $\mathcal{D}^{Obs}$. The server would share the final trained model with the user for local future inference. It is important to note that our threat model protects only image information, not the associated ground truth labels. From our security viewpoint, the obtained information is meaningless if an adversary attempts to extract the visual information from an obfuscated image in $\mathcal{D}^{Obs}$. 

\noindent \textit{Scope of reconstruction attacks:} Our threat model focuses on preventing exact or near-exact recovery of the original visual information from the obfuscated database. We do not consider distributional re-synthesis or semantic regeneration attacks, where an adversary uses strong generative priors, such as super-resolution, hallucination, GAN-based latent search, or diffusion inversion, to synthesize a plausible image consistent with the obfuscated input~\cite{menon2020pulse}. These attacks are important but remain outside the scope of this work and will be evaluated in future extensions of {\em Bit-ViP}.

\noindent \textit{Non-visual channels:}  Although this paper focuses on protecting visual channels by obfuscating the pixels to resist reconstruction/de-identification attacks while maintaining the ML model utility. However, label information and auxiliary metadata (timestamps, device IDs, GPS/EXIF) may contribute to inference risk. It is important to note that these channels are orthogonal to {\em Bit-ViP} and can be mitigated by introducing a label-privacy mechanism, such as label DP (randomized response at the class label) or PATE (noisy aggregation of teacher votes), and by using server-side DP-SGD during training when using a CSP. We make these options explicit and consider them as composable modules with the {\em Bit-ViP}.

\section{Proposed Obfuscation scheme}\label{proposed_method}
This section proposes our {\em Bit-ViP} obfuscation scheme to conceal users' V-PII when image data is stored in a cloud database, while maintaining security. Unlike prior schemes~\cite{lu2020efficient,fu2011novel} that obfuscate images by pixel location permutation, {\em Bit-ViP} follows a novel idea by obfuscating each pixel of the input image via incorporating chaoticness and non-invertible noise without altering its location. {\em Bit-ViP} produces \textit{a non-invertible} image, i.e., the original image cannot be reconstructed from its obfuscated form, thereby protecting it from malicious activities such as reconstruction attacks. \vspace{-0.1in}

\subsection{Bit-ViP Scheme}\label{bit_plane_obfuscation}
Existing image obfuscation schemes obfuscate the secret image as a whole (by adding noise to each pixel intensity value ranging from 0 to 255), which significantly destroys the image's spatial features, such as pixel intercorrelations, and the resulting obfuscated image is least usable for DNN model training. Contrasting prior approaches, we propose an end-to-end scheme to obfuscate an image, say $I$, while preserving its usability by partitioning $I$ into non-overlapping rectangular subimages (a.k.a. blocks). The intuition behind image partitioning to bound the noise range is to be incorporated with the block's pixel intensity range to mitigate the effect of one block's obfuscation on another region of the image. The approach helps obfuscate each non-overlapping region of the secret image with different noise ranges, thereby masking the corresponding spatial features and supporting DNN training. {\em Bit-ViP} obfuscates the input image through \textit{block-based bit-planes obfuscation}. Then, we add Differentially Private (DP) noise to each pixel of the obtained image. Fig.~\ref{proposed_image_obfuscation_scheme} demonstrates the overall obfuscation process of our scheme. We aim to perturb the bit-planes of each sub-image using a chaotic structure that would spread and accumulate to each pixel intensity value during reconstruction via bit-plane combination, without altering pixel locations. Injection of non-invertible, bounded DP-noise into pixels protects image information from vulnerabilities arising from the unstable periodicity of chaotic maps and limited computing precision~\cite{rhouma2010cryptanalysis}.

Assume the dimension of $I$ is $m \times n$ which we partition into $L$ non-overlapping blocks, each of dimension $\tau_1 \times \tau_2$, denoted as $ \left\{I_b^1, I_b^2, \cdots I_b^L\right\}$, where $L = \frac{m \times n}{\tau_1 \times \tau_2} \in \mathbb{N}$. Each block may be perceived as a $\tau_1 \times \tau_2$-dimensional gray-scale image. Given $L$ gray-scale blocks, first, we perform obfuscation on each block as follows.\\\vspace{-0.25cm} 

\noindent \textit{\textbf{Forward step --}} Each $l^{th}$ block, $I_b^l$, contains pixel intensity values in the range $\left[0, 255\right]$ that can be represented by an $8$-bit binary sequence~\cite{zhu2011chaos}, and thus $I_b^l$ can be partitioned into $8$ bit-planes, denoted as $\mathcal{B}\mathcal{P}^{l}_{1}, \mathcal{B}\mathcal{P}^{l}_{2}, \cdots, \mathcal{B}\mathcal{P}^{l}_{8}$, $1 \leq l \leq L$, and each of dimension equal to $I_b^l$, as depicted in Fig.~\ref{stage1}. Mathematically, the bit-planes are computed as
\vspace{-0.3cm}
\begin{equation}
    \mathcal{B}\mathcal{P}_k^l = \mathlarger{\mathlarger{\mathlarger{\lfloor}}}{{\frac{I_b^l}{2^{k-1}}}}\mathlarger{\mathlarger{\mathlarger{\rfloor}}} \textit{ mod } 2 \hspace{0.5cm} \forall \hspace{0.5cm} k = 1,2,..,8
\end{equation} 
where $\lfloor{\cdot}\rfloor$ represents the floor function.

It is well established in the literature~\cite{zhu2011chaos} that almost all $93\%-94\%$ of the total image information is contained in the last four most significant bit-planes. In contrast, the first four least-significant bit-planes contain nearly $6\%-7\%$ only (see Table 1 in~\cite{zhu2011chaos}). Thus, it is worth performing operations for the last four bit-planes, $\mathcal{B}\mathcal{P}^{l}_{5}, \mathcal{B}\mathcal{P}^{l}_{6}, \mathcal{B}\mathcal{P}^{l}_{7}, \mathcal{B}\mathcal{P}^{l}_{8}$, making {\em Bit-ViP} efficient and lightweight.
\begin{figure}[!t]
		\begin{minipage}[]{1.0\linewidth}
			\centering
			\includegraphics[width=1.0\linewidth]{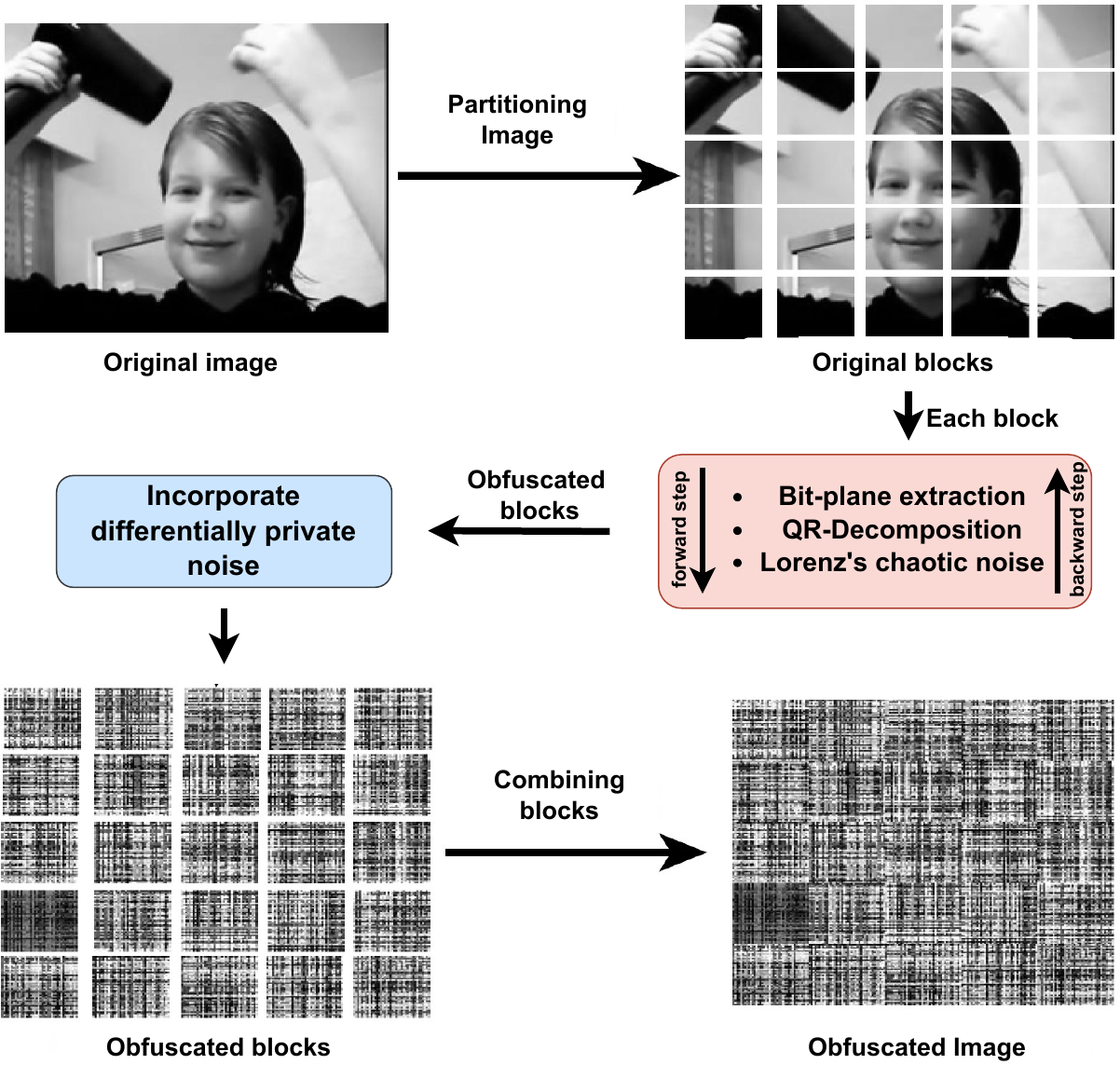}
			\caption{An overview of the proposed image obfuscation scheme, {\em Bit-ViP}.}
		\label{proposed_image_obfuscation_scheme} 
		\end{minipage}
\vspace{-0.1in}
\end{figure}

\noindent Next, we compute {\em QR-decomposition} of $k^{th}$ bit-plane of $I_b^l$ as
\begin{align}\label{qr_decomp}
    \mathcal{B}\mathcal{P}^{l}_{k} = \left[\mathcal{B}\mathcal{P}^{l}_{k}\right]^{Q} \times \left[\mathcal{B}\mathcal{P}^{l}_{k}\right]^{R} \hspace{0.5cm}\forall \hspace{0.5cm} k = 5,6,7,8
\end{align}
where $\left[\mathcal{B}\mathcal{P}^{l}_{k}\right]^{Q}$ and $\left[\mathcal{B}\mathcal{P}^{l}_{k}\right]^{R}$ are \textit{orthogonal} and \textit{upper-triangular} feature maps of dimensions $\tau_1 \times \tau_2$. The QR-decomposition technique is robust to small perturbations, i.e., small perturbations in the QR components result in a variance in the original bit-plane, thereby introducing the desired level of randomness in the generated obfuscated image during reconstruction~\cite{zha1993componentwise}. Here, {\em Bit-ViP} leverages the vectors $[n_1^1, n_1^2]$ and $[n_2^1, n_2^2]$, generated by Lorenz's chaotic structures in Section~\ref{lorenz}, to produce an non-invertible chaoticness for pixel location $(i,j)$ as
\begin{align}\nonumber\label{QR_noise}
noise^Q_{ij} &= n_1^1 + (n_2^1 - n_1^1) \times rand_{ij}^Q \\
noise^R_{ij} &= n_2^1 + (n_2^2 - n_1^2) \times rand_{ij}^R,
\end{align}
where $rand_{ij}^Q$ and $rand_{ij}^R$ are random values generated by Gaussian distribution $\mathcal{N}(0, 1)$, $1 \leq i \leq \tau_1$ and $1 \leq j \leq \tau_2$. Now, we obfuscate each QR component of bit-plane $\mathcal{B}\mathcal{P}^{l}_{k}$ as
\begin{align}\nonumber \label{bp_obs_eq_1}
    \left[\mathcal{B}\mathcal{P}^{l}_{k}\right]_{obs}^{Q}(i, j) &= \left[\mathcal{B}\mathcal{P}^{l}_{k}\right]^{Q}(i, j) + noise^Q_{ij} \\
    \left[\mathcal{B}\mathcal{P}^{l}_{k}\right]_{obs}^{R}(i, j) &=  \left[\mathcal{B}\mathcal{P}^{l}_{k}\right]^{R}(i, j) + noise^R_{ij}.
\end{align}

\begin{figure*}[h]
			\centering	
            \includegraphics[width=1\linewidth,height=6cm]{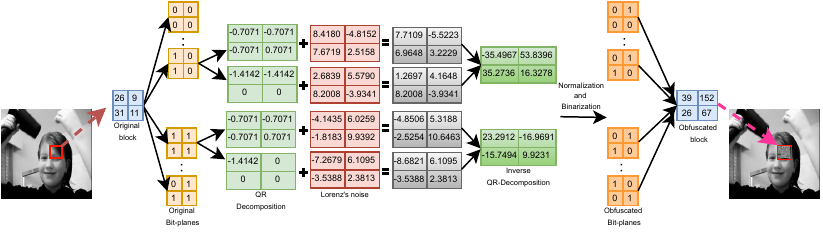}
			\vspace{-0.4in}
			\caption{Intermediate outputs of our proposed scheme for a given block of size $2 \times 2$.}
			\vspace{-0.1in}
		\label{stage1} 
\end{figure*}

\noindent Crucially, the independent and bounded noise is injected at values of $Q$ and $R$, and the reverse composition of noisy components makes a significant perturbation in each reconstructed bit-plane and, hence, the whole block $I_b^l$. Moreover, as $noise^Q_{ij}$ and $noise^R_{ij}$ incorporate the chaotic noise vectors, the obfuscated QR components are now in real-valued space.\\\vspace{-0.25cm}

\noindent \textit{\textbf{Backward step --}} In this step, we perform inverse QR-decomposition, using Eq.~\ref{qr_decomp}, followed by normalization and thresholding to obtain obfuscated bit-planes. Since the results of inverse QR-decomposition contains non-binary real values, {\em Bit-ViP} performs normalization over the range $\left[0, 1\right]$ and then applies binary thresholding at the mean value of the normalized form, which in turn generates the obfuscated bit-plane ${\left[\mathcal{B}\mathcal{P}^{l}_{k}\right]}{_{obs}}$. Normalization scaled the obfuscated bit-plane values to fall within $\left[0, 1\right]$ without distorting differences and maintaining consistent obfuscation quality. In other words, by bringing all images to a similar scale, normalization ensures that the obfuscation process affects each image uniformly. The approach is critical for images with varied lighting conditions or contrast, as it prevents them from being more or less obfuscated by their original pixel value ranges. In contrast, binary thresholding at the mean of the normalized image converts the floating-point values into a binary image (bit-plane). By setting the threshold at the mean for binary thresholding, approximately half the pixels would be above and half below, assuming a normal distribution of pixel values. This ensures that a significant portion of the image's structure and content is retained, which is crucial for maintaining the utility of the obfuscated images for classification tasks. It creates a balance between obfuscating the image (for privacy) and retaining enough structural information (for usability). The conversion to binary effectively obscures finer details and textures that may contain sensitive information. This loss of detail inherently protects privacy, as it prevents the extraction of specific features that could identify individuals or sensitive objects in the image. Further, the normalization and thresholding steps make {\em Bit-ViP} a non-invertible obfuscation function that protects users' V-PII against malicious reconstruction and de-identification attacks originating from the image database. Later, we construct the obfuscated block.
\vspace{-0.25cm}
\begin{align}
    [I_b^l]_{obs} = \sum_{k=1}^{8} [\mathcal{B}\mathcal{P}_k^l]_{obs} \times 2^{k-1}, \hspace{0.25cm} 1\leq l \leq L
 \end{align}
For a block of size $2 \times 2$, Fig.~\ref{stage1} illustrates the intermediate outputs of our proposed scheme, and the pseudo-code for obfuscating a single bit-plane is presented in Algorithm~\ref{algorithm_1}. 

\setlength{\textfloatsep}{1pt}
\SetAlFnt{\small}
\begin{algorithm}[!t]
	\caption{{Bit-plane obfuscation}}
	\label{algorithm_1}
		\KwIn {A bit-plane $\mathcal{B}\mathcal{P}$ of dimension $ \tau_1 \times \tau_2$ and two noise vectors $ [n_1^1, n_1^1]$ and $ [n_2^1, n_2^2]$ generated in Section~\ref{lorenz}}.
		\KwOut {Obfuscated bit-plane $ \mathcal{B}\mathcal{P}_{obs} $}
		\vspace{0.25cm}
		\nonl /* \underline{Pixel obfuscation using QR-decomposition of $\mathcal{B}\mathcal{P}$} */\\
		\vspace{0.25cm}
		$\left[\mathcal{B}\mathcal{P}^{Q}, \mathcal{B}\mathcal{P}^{R}\right]$ = $QR_{Decompose}$({$\mathcal{B}\mathcal{P}$}) \\
		\vspace{0.25cm}
		\For{$ i \leftarrow 1 \hspace{0.1cm} \text{to} \hspace{0.1cm} \tau_1$}{
		\For {$ j \leftarrow 1\hspace{0.1cm} \text{to} \hspace{0.1cm}\tau_2$}{
		$ noise^Q_{i,j} = n_1^q + (n_2^q - n_1^q) \times rand_{i,j}^Q$ \\
		$ noise^R_{i,j} = n_2^r + (n_2^r - n_1^r) \times rand_{i,j}^R$ \\
		\vspace{0.25cm}
		$ \mathcal{B}\mathcal{P}_{obs}^{Q}(i, j) = \mathcal{B}\mathcal{P}^{Q}(i, j) + noise^Q_{i,j}$ \\
		$ \mathcal{B}\mathcal{P}_{obs}^{R}(i, j) = \mathcal{B}\mathcal{P}^{R}(i, j) + noise^R_{i,j}$
		}
		}
		\vspace{0.15cm}
		\nonl /* \underline{Reconstruction of obfuscated bit plane from QR} */ \\
		\vspace{0.05cm}
		$ \mathcal{B}\mathcal{P}_{obs} = \mathcal{B}\mathcal{P}_{obs}^{Q} \times \mathcal{B}\mathcal{P}_{obs}^{R}$\\
		\vspace{0.15cm}
		Compute maximum (Max) and minimum (Min) of $\mathcal{B}\mathcal{P}_{obs}$\\
		\vspace{0.15cm}
		$\mathcal{B}\mathcal{P}_{obs} \leftarrow \frac{\left(\mathcal{B}\mathcal{P}_{obs} - Min \right)}{\left(Max - Min\right)}$ // Normalization \\
		\vspace{0.15cm}
		$\mu = \text{mean}(\mathcal{B}\mathcal{P}_{obs})$ // Threshold for bit-plane Binarization \\
		\vspace{0.15cm}
		\For {$ i \leftarrow 1 \hspace{0.1cm} \text{to} \hspace{0.1cm}\tau_1$}{
		\For {$ j \leftarrow 1\hspace{0.1cm} \text{to} \hspace{0.1cm}\tau_2$}{
		$\mathcal{B}\mathcal{P}_{obs}(i, j) = \begin{cases}0, & \mathcal{B}\mathcal{P}_{obs}(i, j) \leq \mu\\1, &
		\mathcal{B}\mathcal{P}_{obs}(i, j) > \mu \end{cases}$
		}
		}
\end{algorithm}

To this end, we have incorporated a degree of security to preserve the visual privacy of each block's information, and the final obfuscated image can be obtained by combining the obfuscated blocks at their respective positions. However, numerous real-world image scenarios exist where partial protection of the image's visual information is insufficient. For instance, images with large areas of uniform color or brightness can still be discernible post-obfuscation, and high-contrast images with simple geometric features might still be recognizable after obfuscation with small block sizes, resulting in differential and ML-based adversarial attacks. We introduce differential privacy noise into the pixel values of the obfuscated blocks to ensure sufficient contrast in areas of uniform hue or brightness, thereby further concealing the original information.

The most widely used $\epsilon$-DP mechanisms in the research community are {\em{Laplacian}} and {\em{Gaussian}}~\cite{dwork2014algorithmic}, which focus on the privatization of numerical queries by simply adding noise to the pixel intensities themselves. In contrast, {\em Bit-ViP} aims to protect local image features while perturbing the image without compromising the usability of the obfuscated data. Therefore, we utilize the {\em{Exponential}} mechanism~\cite{dwork2014algorithmic}, which allows selecting the ``best" intensity rather than the possible intensities, such as Laplacian and Gaussian mechanisms, with strong data reconstruction guarantees and usability. In other words, the exponential mechanism outputs a pixel intensity value in the range $\left[0, 255\right]$, which may not have the highest score but still significantly preserves privacy.

Given the obfuscated blocks, this section provides an additional layer of security while preserving certain local block features. We inject bounded customized DP noise into each pixel intensity of the obfuscated blocks $\left[I_b^l\right]_{obs}$ obtained above, where $1 \leq l \leq L$. For simplicity, let $\Phi (\cdot)$ denote a function that transforms a plain block $I_b^l$ into an obfuscated one $\left[I_b^l\right]_{obs}$. The traditional notion of DP, Definition~\ref{dp_def}, is defined for two datasets $\mathcal{D}$ and $\mathcal{D'}$ that differ by one element. However, this notion does not apply to two images, as two images differing by one pixel at a single location do not significantly change the visual information. Therefore, we consider the modified definition of DP in the context of our obfuscated image block~\cite {liu2021dp}.

\begin{definition}{\textbf{$\epsilon$-DP-Block:}}\label{dp_block}
	A randomized mechanism $\mathcal{M}$, defined over the dataset of independent and identically distributed (I.I.D.) obfuscated blocks produces $\epsilon$-DP-Block if and only if for any two obfuscated blocks $\left[I_b^l\right]_{obs}$ and $\left[I_b^{l'}\right]_{obs}$ and all subset $U$ of the possible range of perturbed blocks $\mathcal{R}$, the following condition holds
\begin{equation}\nonumber
    Pr\left[\mathcal{M}(\left[I_b^l\right]_{obs})\in U\right]\leq \text{exp}(\epsilon)\hspace{0.05cm}\cdot Pr\left[\mathcal{M}(\left[I_b^{l'}\right]_{obs})\in U\right],
\end{equation}
\noindent where $\epsilon$ quantifies the amount of privacy offered by $\mathcal{M}$. In our definition, IID means blocks of the same dimensions.
\end{definition} 
Traditionally, DP noise (e.g., {\em{Exponential}}) is added locally to the information before it is transmitted to the cloud server. Such mechanisms are well-suited for decentralized learning in which DP-noise can be added to the model (instead of to the data) before dispatching it to the cloud server, as in federated learning~\cite{zhou20222d}. In our scenario, the clients share the obfuscated data (i.e., images) with the server without compromising V-PII. Thus, we need to develop a reconstruction mapping, denoted by $\Psi$, that converts the DP-block $\mathcal{M}(\left[I_b\right]_{obs})$ back to the previous step domain, ensuring that the reconstructed image preserves visual information. The $\epsilon-$DP step provides a privacy guarantee that holds under any data-independent transformations of the obfuscated output. This paper addresses each reconstruction resistance and pixel structure recovery in the visual channel. We apply the Exponential Mechanism in the 8-bit intensity domain; using a normalized $\left[0, 1\right]$ domain would rescale the sensitivity ($\Delta$ ref. Def. \ref{sensitivity_def}) to $255$ without affecting the DP guarantee.
\begin{definition}{\textbf{Sensitivity in obfuscated block:}}\label{sensitivity_def}
Let $\mathcal{R}$ denote the possible perturbed blocks of the same dimensions, having pixel intensity values in the range $\left[0, 255\right]$ over the set of obfuscated blocks $\mathcal{D}$. Define a score function $q: \mathcal{D} \times \mathcal{R} \rightarrow \mathbb{R}$ to quantify the output quality $r \in \mathcal{R}$ for an input block. Then, the sensitivity, $\Delta$, is defined as the maximum difference $q$ produced by two blocks, $q(\left[I_b^l\right]_{obs})$ and $q(\left[I_b^{l'}\right]_{obs})$, \hspace{0.1cm} $l\neq l'$:
\vspace{-0.2cm}
\begin{align}
\Delta
&= \max_{r\in\mathcal{R}}
   \Bigl|\, q\!\big(\left[I_b\right]_{\mathrm{obs}},r\big)
         - q\!\big(\left[I_b'\right]_{\mathrm{obs}},r\big) \Bigr| \nonumber\\
&\le \max_{r\in\mathcal{R}}
   \bigl\|\left[I_b\right]_{\mathrm{obs}}-\left[I_b'\right]_{\mathrm{obs}}\bigr\|_{1} \le 255\,T,\qquad T=\tau_1\tau_2 \nonumber
\end{align}

where, $q(\left[I_b^{l}\right]_{obs}, r) = -\bigl\| r - \left[I_b^{l}\right]_{\mathrm{obs}} \bigr\|_{1}$
\end{definition}
\noindent Next, we state an exponential mechanism in Theorem~\ref{theorem_1} and prove that it satisfies $\epsilon$-DP conditions for the probability distribution $\text{Pr}[r]\sim exp\left(\frac{\epsilon q\left(\mathcal{D}, r\right)}{2\Delta}\right)$, which signifies that the obfuscated block reveals negligible spatial information.
\begin{theorem}{\textbf{$\left(\epsilon, \tau_1 \times \tau_2\right)$-DP-Block Exponential Mechanism:}}\label{theorem_1} For a given obfuscated block $\left[I_b\right]_{obs}$ of dimension $\tau_1\tau_2=T$, a mechanism $\mathcal{M}_{E}$ that maps $\left[I_b\right]_{obs}$ to a perturbed block $\left[I_b\right]_{obs}^2 \in \mathcal{R}$, is $\left(\epsilon,  \tau_1 \times \tau_2\right)$-DP-Block if and only if
    $\mathcal{M}_{E}(\left[I_b\right]_{obs})$ is equal to $\Psi\left[\Phi(I_b)+\mathcal{N}_{Exp}\right]\hspace{0.2cm}\forall \hspace{0.1cm} \left[I_b\right]_{obs} \in \mathcal{D}$,
where $\Psi$ reconstructs the obfuscated block to a block with values in $\mathcal{R}$, $\mathcal{N}_{Exp}$ denotes an exponential mechanism, and $\mathcal{D}$ contains blocks of dimensions $\tau_1 \times \tau_2$.
\end{theorem}
\begin{proof}
We define a bounded range $\Psi = \left[0, \Gamma\right]$ to avoid the excess incorporation of DP-noise to pixels in each block; otherwise, the obfuscated image as a whole may significantly drop the training performance of the underlying DNN model, where 
\vspace{-0.1in}
\begin{equation}\nonumber
    \Gamma = \max_{1\leq l \leq L} \left\lfloor\dfrac{max\left(\left[I_b^l\right]_{obs}\right) - min\left(\left[I_b^l\right]_{obs}\right)}{2}\right\rfloor,
\vspace{-0.1in}
\end{equation}
\noindent and $\lfloor.\rfloor$ is a floor function.
\noindent To prove the stated mechanism $\mathcal{M}_{E}$ holds $\epsilon$-DP, it is sufficient to show that the output distribution of the function $\mathcal{M}$, defined below, satisfies $\epsilon$-DP (Definition~\ref{dp_block}),

\begin{align*}\nonumber
    \mathcal{M}(I_b) = \left\{\left[\Phi(I_b) + \mathcal{N}_{Exp}\right] \cdot 10^{\beta}\right\} \ mod \  \Psi\\
    &\hspace{-5.2cm}=\left\{\left[\left[I_b\right]_{obs} + \mathcal{N}_{Exp}\right] \cdot 10^{\beta}\right\} \ mod \  \Psi,  \hspace{0.2cm} \forall \hspace{0.1cm} \beta \in \mathbb{N}
\end{align*}

The exponential output distribution of $\mathcal{M}(I_b)$ is close in a multiplication sense to $\mathcal{M}(I_b')$ everywhere. Further steps of the proof are followed by post-processing properties of DP: finite addition and/or composition of $\epsilon$-DP mechanisms are also DP. We adopt a few notations for convenience: $X=I_b$, $X'=I'_b$, and $T=\tau_1 \times \tau_2$. Consider, 
\begin{small}
\begin{align}\nonumber
\frac{Pr\left(\text{Exp}(X,\mathcal{R},q,\epsilon/T)=r\right)}{Pr\left(\text{Exp}(X',\mathcal{R},q,\epsilon/T)=r\right)}\\ \nonumber & \hspace{-3.3cm}= \prod_{t=1}^T\frac{Pr\left(\text{Exp}(X_t,\mathcal{R},q,\epsilon/T)=r\right)}{Pr\left(\text{Exp}(X'_t,\mathcal{R},q,\epsilon/T)=r\right)}\\ \nonumber
& \hspace{-3.3cm} = \prod_{t=1}^T \frac{exp\left(\frac{\epsilon q(X_t, r)}{2\Delta T}\right)}{\sum_{r' \in \mathcal{R}}exp\left(\frac{\epsilon q(X_t, r')}{2\Delta T}\right)}\cdot \frac{\sum_{r' \in \mathcal{R}}exp\left(\frac{\epsilon q(X'_t, r')}{2\Delta T}\right)}{exp\left(\frac{\epsilon q(X'_t, r)}{2\Delta T}\right)}\\ \nonumber
& \hspace{-3.3cm}= \prod_{t=1}^T \frac{exp\left(\frac{\epsilon q(X_t, r)}{2\Delta T} \right)}{exp\left(\frac{\epsilon q(X'_t, r)}{2\Delta T}\right)} \cdot \frac{\sum_{r' \in \mathcal{R}} exp\left(\frac{\epsilon q(X'_t, r')}{2\Delta T}\right)}{\sum_{r' \in \mathcal{R}}exp\left(\frac{\epsilon q(X_t, r')}{2\Delta T}\right)}\\ \nonumber
& \hspace{-3.3cm}\leq \prod_{t=1}^T exp\left(\frac{\epsilon}{2\Delta T}\left(q\left(X_t,r\right) -q\left(X'_t,r\right)\right)\right) \cdot \\ \nonumber& \hspace{-3cm}\frac{\sum_{r' \in \mathcal{R}}exp\left(\frac{\epsilon q(X_t, r')+\Delta}{2\Delta T}\right)}{\sum_{r' \in \mathcal{R}}exp\left(\frac{\epsilon q(X_t, r')}{2\Delta T}\right)}\\ \nonumber
& \hspace{-3.3cm} = \prod_{t=1}^T exp\left(\frac{\epsilon}{2T}\right) \cdot exp\left(\frac{\epsilon}{2T}\right) \cdot \frac{\sum_{r' \in \mathcal{R}}exp\left(\frac{\epsilon q(X_t, r')}{2\Delta T}\right)}{\sum_{r' \in \mathcal{R}}exp\left(\frac{\epsilon q(X_t, r')}{2\Delta T}\right)}\\ \nonumber
& \hspace{-3.3cm} = \prod_{t=1}^T exp\left(\frac{\epsilon}{T}\right) = exp(\epsilon),
\end{align}
\end{small}

\noindent where the inequality holds from the sensitivity definition. By symmetry, we get $\frac{Pr\left(\text{Exp}(X,\mathcal{R},q,\epsilon/T)=r\right)}{Pr\left(\text{Exp}(X',\mathcal{R},q,\epsilon/T)=r\right)}\geq exp(-\epsilon)$ which proves the mechanism $\mathcal{M}$ is $\epsilon$-DP. 
\end{proof}
\vspace{-0.05in}

After performing obfuscation over all $L$ blocks, we concatenate all the obfuscated blocks $\left\{\left[I_B^{1}\right]^2_{obs}, \left[I_B^{2}\right]^2_{obs}, ...  \left[I_B^{L}\right]^2_{obs}\right\}$ at the corresponding locations of $\left\{I_B^1, I_B^2, ... I_B^L\right\}$ to obtain our final obfuscated image $I_{obs}$. In Algorithm~\ref{algorithm_2}, we presented the pseudo-code of {\em Bit-ViP} for obfuscating a gray-scale image. We independently apply {\em Bit-ViP} to all color channels of an RGB image.

\SetAlFnt{\small}
\begin{algorithm}[h]
	\caption{{\em Bit-ViP} scheme}
	\label{algorithm_2}
		\KwIn {A $m \times n$-dimensional gray-scale image $I$}
		\KwOut {Obfuscated image $ I_{obs} $}
		Partition $I$ into blocks of size $\tau_1 \times \tau_2$. Let $\mathbf{I}_{b} = \left[I_b^1, I_b^2, ... I_b^L\right]$ are obtained blocks in row-roster form. \\
		\vspace{0.25cm}
		\For {each block $I_b^l \in \mathbf{I}_{B}, \forall \hspace{0.1cm} l \in \{1,2,\cdots, L\}$ }{
		$[\mathcal{B}\mathcal{P}^{l}_{1}, \mathcal{B}\mathcal{P}^{l}_{2}, \cdots, \mathcal{B}\mathcal{P}^{l}_{8}] \gets$ Extract bit-planes of $I_b^l$ \\
		\For {$ t \leftarrow 5\hspace{0.1cm} \text{to} \hspace{0.1cm}8$}{
		Obfuscate $\mathcal{B}\mathcal{P}^l_t$ using Algorithm \ref{algorithm_1} to get $[\mathcal{B}\mathcal{P}^l_t]_{obs}$
		}
		Construct obfuscated block $\left[I_b^{l}\right]_{obs}$ using $\left\{[\mathcal{B}\mathcal{P}^l_t]_{obs}\right\}_{t=1}^{8} $\\
		\nonl /* Add exponential noise to each pixel of $\left[I_b^{l}\right]_{obs}$ */\\
		\For {$ i \leftarrow 1\hspace{0.1cm} \text{to} \hspace{0.1cm}\tau_1$}{
		\For {$ j \leftarrow 1\hspace{0.1cm} \text{to} \hspace{0.1cm}\tau_2$}{
		$\left[I_b^{l}\right]_{obs}^2(i, j) \leftarrow \left[I_b^{l}\right]_{obs}(i, j) + \mathcal{N}_{Exp}, \hspace{0.1cm}$
		}
		}
		}
		Construct obfuscated image $I_{obs} \gets \text{Concatenate}  \left\{\left[I_b^{l}\right]^2_{obs}\right\}_{l=1}^L$
\end{algorithm}

\begin{figure}[h]
		\begin{minipage}[b]{\columnwidth}
			\centering
			\includegraphics[width=1.0\linewidth]{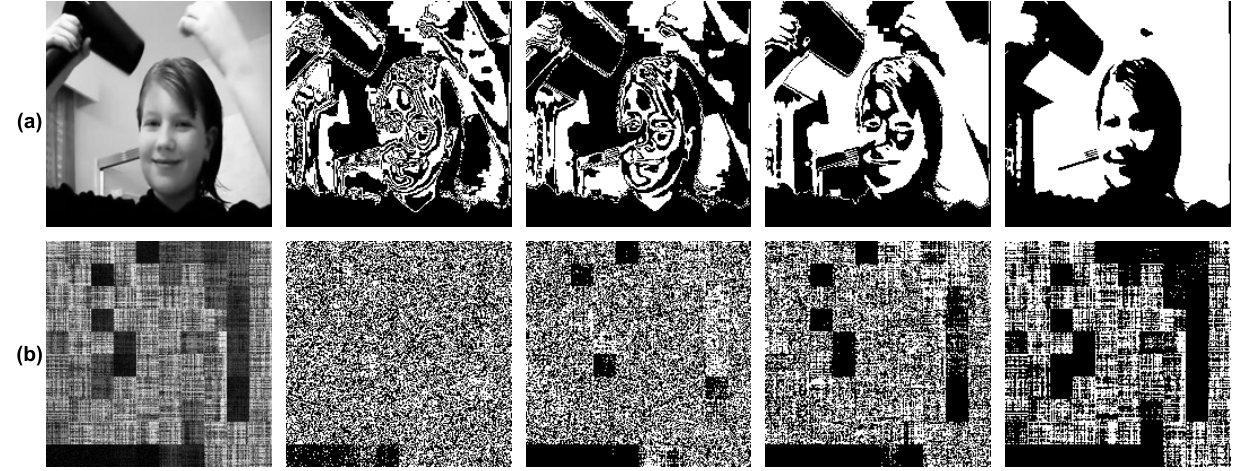}
			\vspace{-0.5cm}
			\caption{Four most significant bit-planes comparison of (a) original image and (b) obfuscated image obtained upon applying {\em Bit-ViP}.}
			\label{bit_plane_compare} 
		\end{minipage}
\end{figure}

To demonstrate the change in V-PII, we show the four most significant bit-planes of an original image (of hair-blowing activity) and its obfuscated form in Fig.~\ref{bit_plane_compare}. To analyze the information security of {\em Bit-ViP}, we compute the number of random values used by our scheme to transform the input image $I$ of dimension $m \times n$ into $I_{obs}$ of the same dimension; the more random values, the higher the security.

\begin{corollary}{\textbf{Obfuscated image is $\epsilon$-DP:}} For a given image $I$ that can be partitioned into $L$ non-overlapping blocks, the obfuscated form, $I_{obs}$, obtained using {\em Bit-ViP} follows $\epsilon$-DP Exponential mechanism.
\end{corollary}
\begin{proof}
In our final step, we concatenate the obfuscated blocks to obtain $I_{obs}$, which is equivalent to the parallel composition of various DP mechanisms over disjoint, independent blocks. Also, we proved in Theorem~\ref{theorem_1} that each of the $L$ blocks follows the $\epsilon$-DP Exponential mechanism. Therefore, following the post-processing properties in DP~\cite{dwork2014algorithmic}, $I_{obs}$ follows $\epsilon$-DP Exponential mechanism.
\end{proof}
\vspace{-0.1cm}
\begin{lemma}\label{lemma_1} Assume that an $m\times n$-dimensional image $I$ can be partitioned into $L$ non-overlapping blocks, each of dimension $\tau_1 \times \tau_2$. Then, the total random values incorporated to obfuscate $I$ by our proposed scheme is
\begin{equation}\nonumber
\vspace{-0.05cm}
    32L+16L\tau_1\tau_2 + m \times n, \hspace{0.25cm} where \hspace{0.25cm} L = \frac{m \times n}{\tau_1 \times \tau_2}
\end{equation}
\end{lemma}
\vspace{-0.1cm}
\begin{proof}
Ideally, the probability of reconstruction of any pixel intensity value $I(i,j)$ of image $I$ at pixel-location $(i,j)$ using the obfuscated intensity value $I_{obs}(i,j)$ should be close to zero. Given the block size $\tau_1 \times \tau_2$, we need to obfuscate $L$ blocks to obtain $I_{obs}$. Each block has 8 bit planes, so the total number of bit-plane obfuscations is $L \times 8$.

In the algorithm, we initialize each of $L \times 8$ bit-planes with two noise vectors, as shown in Eq.~\ref{QR_noise}, generated by the solution matrix of Lorenz's chaotic system (see Eq.~\ref{chaotic_system_eqs}). Thus, we used $4 \times L \times 8$ random numbers. Further, we decomposed each bit-plane into two $\tau_1 \times \tau_2$-dimensional QR components, and thus, the total number of random values is
\vspace{-0.2cm}
\begin{equation}\nonumber
    \begin{split}
        \mathcal{R}\mathcal{V}_{1} = 4 \times L \times 8 + 2 \times L \times 8 \times \tau_1 \times \tau_2 = 32L +16L\tau_1\tau_2.
    \end{split}
\end{equation}
\noindent Further, we added a differentially private random number to each pixel intensity in each obfuscated block, yielding $m \times n$ additional random values in $\mathcal{R}\mathcal{V}_1$. Thus, the total random numbers to obfuscate $I$ are 
\vspace{-0.2cm}
\begin{equation}\nonumber
\vspace{-0.1cm}
    \begin{split}
        Total_{RV} &= 32L +16L\tau_1\tau_2 + m \times n.
    \end{split}
\vspace{-1.0cm}
\end{equation}
Hence proved Lemma~\ref{lemma_1}.
\end{proof}

Note that the random values counted for $\mathcal{R}\mathcal{V}_{1}$ distributed uniformly over the unit interval $[0, 1]$. Therefore, the probability for getting a random value becomes $\frac{1}{|\mathcal{R}\mathcal{V}_{1}|}$, where $|\cdot|$ computes the size. Since these random values are generated by chaotic structures, {\em Bit-ViP} would discourage adversaries from recovering original pixel values. Moreover, the last $m \times n$ random values are distributed across $L$ distinct exponential-differentially private intervals. Based on this analysis, we assert that an adversary can extract almost no spatial information about the original pixel intensities from the obfuscated pixel values, even with complete access to the obfuscated database. For an RGB-color image, the total number of random values is $Total_{RV} \times 3$, which means {\em Bit-ViP} would offer enhanced security. 

\begin{theorem}{\textbf{Security is directly proportional to block size:}}\label{theorem_2} Let $J_1$ and $J_2$ be two image blocks of dimensions $\tau_1 \times \tau_2$ and $\eta_1 \times \eta_2$ respectively, where $\eta_1 > \tau_1$ and $\eta_2 > \tau_2$. Then, the obfuscated form of $J_2$ obtained by the {\em Bit-ViP} scheme is more secure than that of $J_1$. 
\end{theorem}
\begin{proof}
Let us assume that the obfuscated form of $J_1$ and $J_2$ are ${J_1}_{obs}$ and ${J_2}_{obs}$ respectively. We prove our theorem by showing that the total number of random values utilized in ${J_2}_{obs}$ is more than that incorporated in ${J_1}_{obs}$.

\noindent As explained in Lemma~\ref{lemma_1}, the total number of random values incorporated to obtain ${J_1}_{obs}$ and ${J_1}_{obs}$ are \vspace{-0.15cm}
\begin{equation}\nonumber
    \begin{split}
        {Total_{RV}}_1 &= 32 +17\tau_1\tau_2,\\
        {Total_{RV}}_2 &= 32 +17\eta_1\eta_2
    \end{split}
\end{equation}
\vspace{-0.2cm}
Consider, 
\begin{equation}\nonumber
    \begin{split}
        {{Diff}_{RV}} &= {Total_{RV}}_2 - {Total_{RV}}_1\\
             &= 32 +17\eta_1\eta_2 - 32 +17\tau_1\tau_2\\
             &= 17\left(\eta_1\eta_2 - \tau_1\tau_2\right) > 0\\
& \hspace{2cm}\text{as} \hspace{0.1cm} \eta_1 > \tau_1 \hspace{0.1cm} \text{and} \hspace{0.1cm} \eta_2 > \tau_2
    \end{split}
\end{equation}
Since, ${Diff_{RV}}>0$ which proves our Theorem~\ref{theorem_2}.
\end{proof}
\vspace{-0.2in}
\subsection{Qualitative analysis of {\em Bit-ViP} scheme}\label{sec_5}
This section qualitatively analyzes the information security of the obfuscated image obtained from our scheme across varying block sizes. As mentioned in our scheme, the input image is divided into non-overlapping blocks of dimension $\tau_1 \times \tau_2$. Here, we visually show the obfuscated images obtained with varying block sizes: $5 \times 5$, $10 \times 10$, $20 \times 20$, $25 \times 25$, and $40 \times 40$, for the original image of {\em hair blowing} activity of dimension $200\times 200$ in Fig. \ref{proposed_obfuscation_schemes_block_comparison}. For instance, an obfuscated image with a block size of $5 \times 5$ reveals a few edges and activity attributes, whereas one with a block size of $25 \times 25$ conceals them. It can be observed that the visually sensitive information would not be revealed by an adversary from the obfuscated images (stored in the image database in the cloud) for large block sizes; however, it comes at the cost of degraded accuracy, which we validated empirically, and the results are reported in the next section.
\begin{figure}[h]
		\begin{minipage}[b]{\columnwidth}
			\centering
			\includegraphics[width=0.98\linewidth,height=5cm] {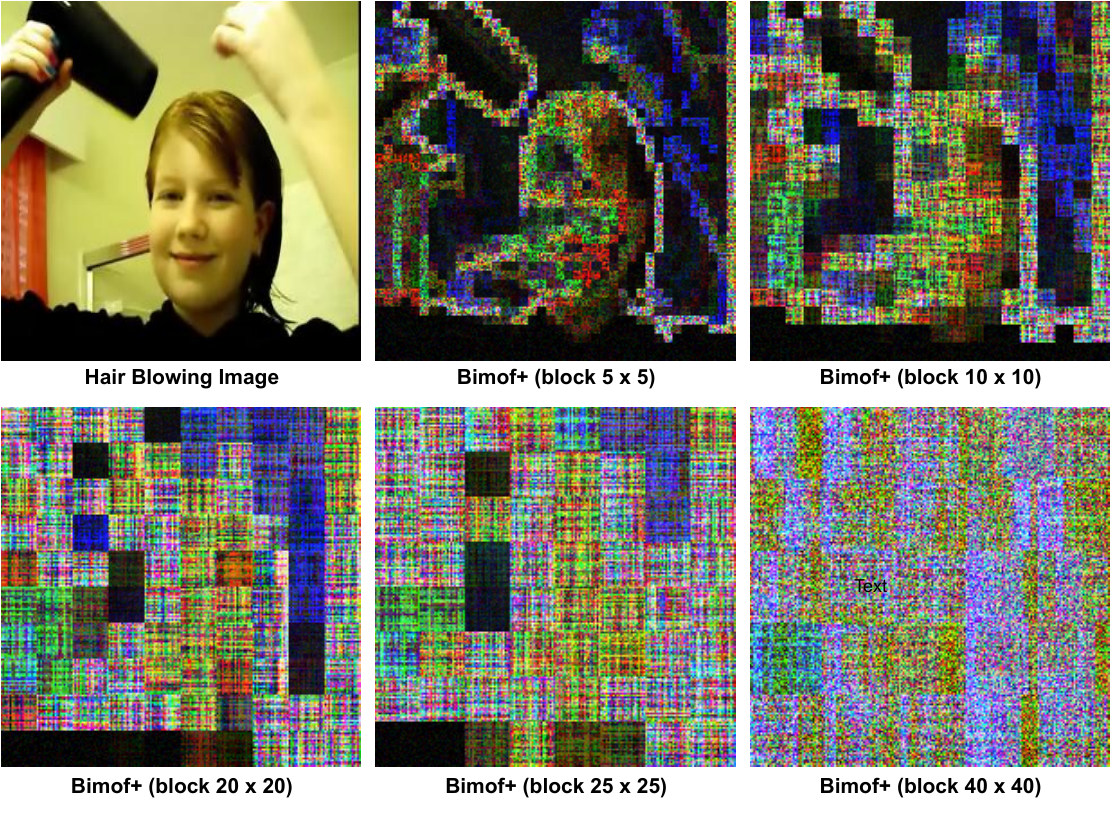}
			\caption{Comparison of {\em Bit-ViP} obfuscated images with different block sizes. A larger block size offers better privacy.}
			\label{proposed_obfuscation_schemes_block_comparison} 
		\end{minipage}
\end{figure}

\section{Experimental Evaluation}\label{experi_results}
We validate the effectiveness of the proposed scheme against the activity recognition system using performance metrics, including recognition accuracy on benchmark video datasets, namely \textit{HMDB51} \cite{kuehne2011hmdb} and \textit{UCF101} \cite{soomro2012ucf101}. This section reports and critically analyses the results obtained by extensive experiments. 

We implemented the proposed scheme in {\em Python} language on Ubuntu $20.14$ 64-bit, over an HP workstation with an Nvidia Quadro P5000 graphics card and an Intel Xeon(R) 5120 CPU @$2.20\times$56 GHz. Four state-of-the-art DCNNs, namely \textit{ResNet18}, \textit{ResNet34}, \textit{ResNet50}, and \textit{VGG16} are employed for activity recognition system. Training hyperparameters: batch size $ = 128$, epochs $= 125$, optimizer $= adam$, cosine annealing lr-scheduler, learning rate $= 1e^{-3}$, and negative log-likelihood loss. While maintaining the original train-test split ratio, we use $10\%$ of the training data for validation. The size of each block is a factor of $200\times200$ (the dimension of the original frame). We assume privacy budget $\epsilon = 0.5$ as considered in~\cite{fan2018image}. We also provide an accuracy comparison with varying $\epsilon$ in the supplementary file.

To provide a fair comparison, we compare {\em Bit-ViP} with three axes: (i) \textit{security} (resistance to reconstruction/de-identification attacks), (ii) \textit{utility} (recognition accuracy on obfuscated data), and (iii) \textit{visual/structural distortion}, and the baseline work is considered accordingly.

\subsection{Dataset Description}\label{dataset}
\begin{enumerate}
	\item \textbf{HMDB51} dataset \cite{kuehne2011hmdb} consists of $6849$ realistic video clips with $51$ categories (classes) of human activities, and there exist more than $100$ clips for each class; some of the activities are ``throw", ``pull-ups", ``pick"{\em, etc.} 
	\item \textbf{UCF101} \cite{soomro2012ucf101} is also a standard dataset for evaluating HAR systems. It consists of $13320$ video clips divided into $101$ classes, like ``blow dry hair", ``push-ups" {\em, etc.} 
\end{enumerate}
\vspace{-0.15cm}
\subsection{Recognition Accuracy}\label{recognition_ana_ED}
At first, we conducted experiments to evaluate the recognition accuracy of the considered models in the plain and obfuscated data domains, and the obtained results are reported in Fig.~\ref{recog_accu_stages}. Our scheme is model-agnostic, meaning it works with different models without requiring any change in the underlying process/steps. It addresses the {\em{model architecture dependency}} limitation of the existing schemes, highlighted in Section~\ref{intro}. Moving ahead with the results, all the models lose accuracy by an admissible magnitude, approximately $16\%-22\%$, on obfuscated data compared to the ones with plain data and $6\%-18\%$ as compared to Bimof~\cite{tanwar2023preserving}, but with much stronger security because of incorporating the DP-block exponential mechanism. The magnitude of the difference depends on the chosen block size during obfuscation. For instance, the best-performing model {\em ResNet50} shows an accuracy drop of $\sim 16.2\%$ from plain data to the least obfuscation block $5\times5$ and $27\%$ from block $5\times5$ to $40\times40$, for the HMDB51 dataset. The reason for such a drop is elevated perturbations in spatial information with a larger block size, breaking the inter-correlation among neighboring pixels and eventually affecting the model's learning, thereby addressing the limitations of {\em{security and usability}}. Still, it paid off in terms of more robust protection of visual information. Thus, our scheme enables users to choose an appropriate block size based on the level of protection for personal information they desire, finally addressing the limitation of the need for a {\em{trusted server}}. We observed a similar privacy-utility trade-off reported in prior vision research work that leverage strong, pre-trained features for DP learning, i.e., modest utility gaps even at low $\varepsilon$ due to robust features and transfer (e.g., \(\sim\)84.3\% with \(\varepsilon=0.1\) and \(\sim\)88\% with \(\varepsilon=8\))~\cite{mehta2022differentially}.
\begin{figure*}[!t]
\centering
\minipage{0.24\textwidth}
\centering
    \includegraphics[scale=0.28]{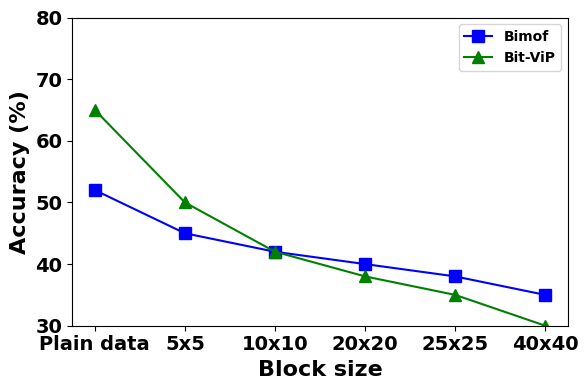}
    \subcaption{\footnotesize{VGG16 results on HMDB51}}
\endminipage
\minipage{0.24\textwidth}
\centering
    \includegraphics[scale=0.28]{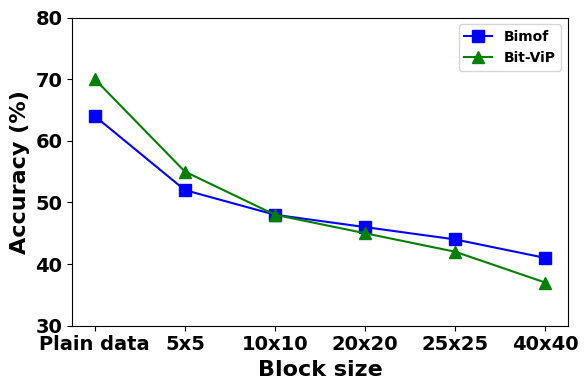}
    \subcaption{\footnotesize{ResNet18 results on HMDB51}}
\endminipage
\minipage{0.24\textwidth}
\centering
    \includegraphics[scale=0.28]{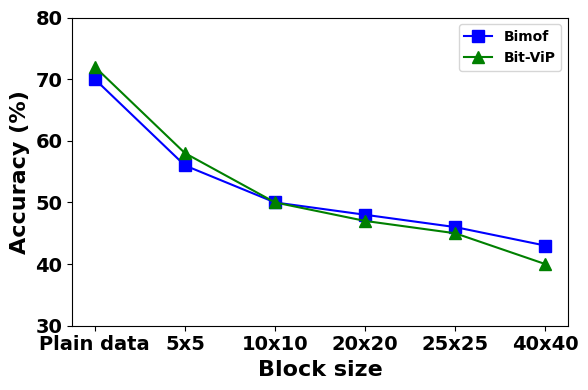}
    \subcaption{\footnotesize{ResNet34 results on HMDB51}}
\endminipage
\minipage{0.24\textwidth}
\centering
    \includegraphics[scale=0.28]{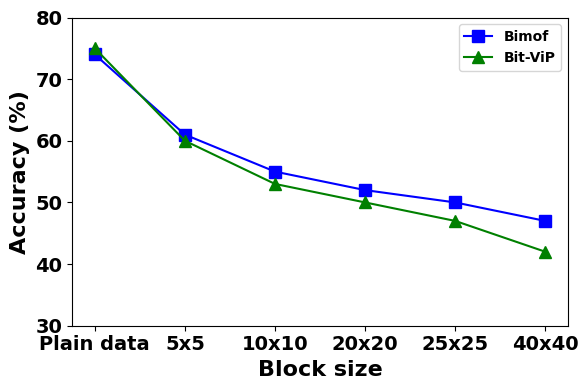}
    \subcaption{\footnotesize{ResNet50 results on HMDB51}}
\endminipage
\vspace{0.05in}
\centering
\minipage{0.24\textwidth}
\centering
    \includegraphics[scale=0.28]{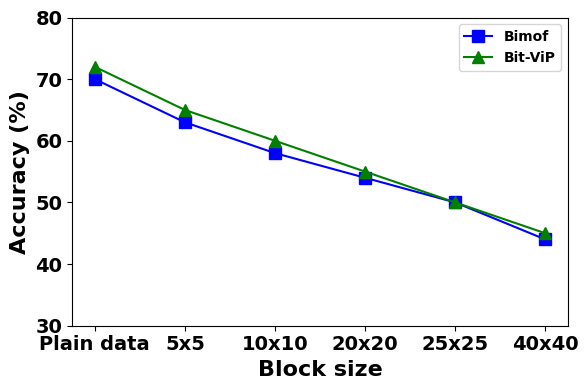}
    \subcaption{\footnotesize{VGG16 results on UCF}}
\endminipage
\minipage{0.24\textwidth}
\centering
    \includegraphics[scale=0.28]{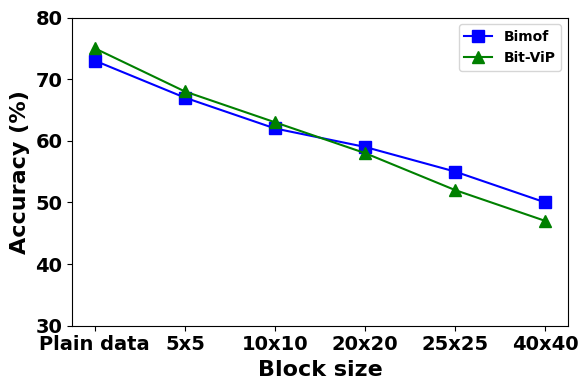}
    \subcaption{\footnotesize{ResNet18 results on UCF}}
\endminipage
\minipage{0.24\textwidth}
\centering
    \includegraphics[scale=0.28]{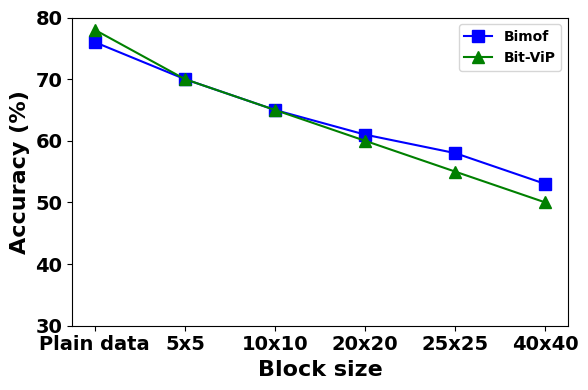}
    \subcaption{\footnotesize{ResNet34 results on UCF}}
\endminipage
\minipage{0.24\textwidth}
\centering
    \includegraphics[scale=0.28]{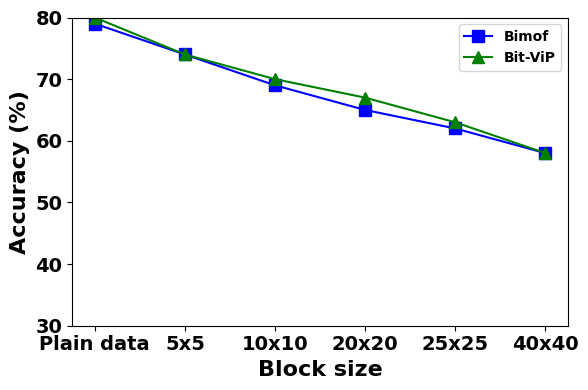}
    \subcaption{\footnotesize{ResNet50 results on UCF}}
\endminipage
\caption{Accuracy ($\%$) on plain and obfuscated data with varying block sizes.}
\label{recog_accu_stages}
\end{figure*}

Next, we compare the accuracy of the {\em Bit-ViP} scheme with prior obfuscation schemes~\cite{zhou2022new,kim2019privacy,fan2018image,rajput2020privacy,imran2020robust} using the best-performing model, {\em ResNet50}, in Table~\ref{accuracy_comparison}. It is easy to see that the down-sampling scheme~\cite{kim2019privacy} achieves higher accuracy than the proposed scheme but provides no security, making it vulnerable to adversarial attacks such as reconstruction and model inversion. On the other hand, a few prior schemes ensure security but suffer a drastic drop in accuracy when applied to obfuscated data, thus reducing usability.   

Our previous approach, Bimof, outperforms {\em Bit-ViP}; it optimizes only task-specific accuracy rather than the privacy-utility trade-off. In contrast, {\em Bit-ViP} offers strong security while maintaining usability by achieving adequate recognition accuracy while preserving the location of pixels in the obfuscated image.

\begin{table}[!h]
    \setlength{\tabcolsep}{6.75pt} 
    \renewcommand{\arraystretch}{1.2} 
	\centering
	\caption{Accuracy (i.e., usability) and security comparison of different schemes. Unlike prior schemes, which offer either security or usability, {\em Bit-ViP} offers strong security while maintaining its usability by achieving an adequate level of recognition accuracy.}
	\resizebox{0.48\textwidth}{!}{
    \begin{tabular}{|l|c|c|c|c|}\hline
        \multicolumn{1}{|c|}{\textbf{Scheme}} & \textbf{HMDB51} & \textbf{UCF101} & \textbf{Security} & \textbf{Usable} \\ \hline
        \textbf{Original} & 68.7 & 79.1 & $\times$  & \textbf{${\checkmark}$} \\ \hline
        \textbf{Face blurring\cite{kim2019privacy}} & 67.6 & 78.4 & $\times$ & $\checkmark$ \\ \hline
        \textbf{Full mosaicing\cite{chou2018privacy}} & 66.2 & 76.7 & $\times$ & $\checkmark$ \\ \hline
        \textbf{Down-sampling \cite{ryoo2018extreme}} & 64.4 & 74.6 &$\times$ & $\checkmark$\\ \hline 
        \textbf{Pixelation \cite{fan2018image}} & 18.9 & 21.9 & $\checkmark$& $\times$ \\ \hline
        \textbf{Encryption~\cite{zhou2022new}} & 10.2 & 14.5 & $\checkmark$ & $\times$ \\ \hline
        \textbf{Scrambling \cite{jeevitha2021novel}} & 11.2 &  14.0 & $\checkmark$ & $\times$ \\ \hline
        \textbf{Superpixel Noise~\cite{rajput2020privacy}} & 15.4 & 19.6 & $\checkmark$& $\times$ \\ \hline
        \textbf{{\em Bimof} (block $5 \times 5$)~\cite{tanwar2023preserving}} & 65.8 & 77.1 & $\checkmark$ & $\checkmark$ \\ \hline
        \textbf{{\em Bit-ViP} (block $5 \times 5$)} & 52.5 & 71.1 & $\checkmark$ & $\checkmark$ \\ \hline
        \textbf{{\em Bimof} (block $40 \times 40$)~\cite{tanwar2023preserving}}& 59.8  & 68.3 & $\checkmark$ & $\checkmark$ \\ \hline
        \textbf{{\em Bit-ViP} (block $40 \times 40$) }& 41.1  & 56.7 & $\checkmark$ & $\checkmark$ \\ \hline
    \end{tabular}}
    \label{accuracy_comparison}
\end{table}

\noindent $\bullet$ \textbf{Activity-wise accuracy:} We selected $10$ activities from each dataset and reported the obtained accuracy in Fig.~\ref{accu_bar}. The selected activities from the HMDB51 dataset are Brushing hair (A1), Dribbling (A2), Drinking (A3), Ride Horse (A4), Run (A5), Shoot Ball (A6), Sitting (A7), Throwing (A8), Turn (A9), and Walk (A10). From the UCF101 database, $10$ activities are Horse Riding (B1), Bench Press (B2), Biking (B3), Playing Dhol (B4), Salsa spin (B5), Throwing disc (B6), Soccer juggling (B7), Drumming (B8), Rowing (B9), and Punch (B10). Regardless of the activity type and the dataset, we observe an accuracy drop of $8\% \sim 10\%$ when the model learns from obfuscated data with a block size of $5 \times 5$ compared to plain data, and this drop increases by $10\% \sim 15\%$ with a block size of $40 \times 40$. However, strong visual privacy pays off from such a drop. Note that this paper focuses on developing a novel privacy-preserving scheme that bridges privacy and usability for an image database hosted on the cloud server, rather than solely improving usability (accuracy).

\begin{figure}[h]
\centering
\minipage{0.49\textwidth}
\centering
    \includegraphics[scale=0.65]{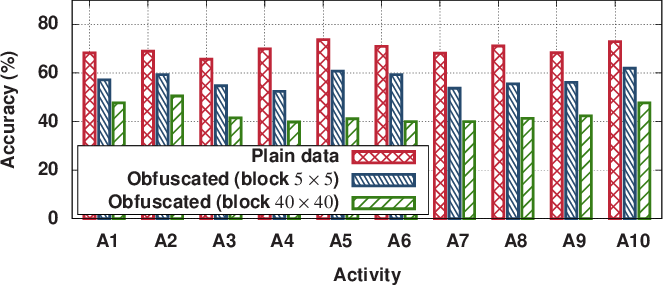}
    \subcaption{\footnotesize{On HMDB51 dataset}}
\endminipage\hfill
\vspace{0.2cm}
\minipage{0.49\textwidth}
\centering
    \includegraphics[scale=0.65]{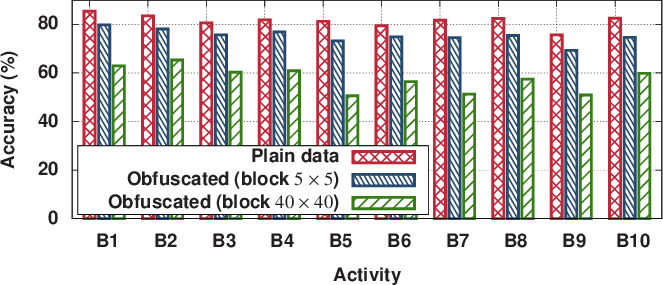}
    \subcaption{\footnotesize{On UCF101 dataset}}
\endminipage\hfill
 \caption{Accuracy for $10$ randomly selected activities.}
\label{accu_bar}
\end{figure}

\begin{table*}[h]
	\centering
	\caption{Reference- and non-reference-based image structural analysis.}
\begin{tabular}{|l|cccc|ccc|cccc|}
\hline
 & \multicolumn{4}{c|}{\textbf{Referenced Based}} & \multicolumn{3}{c|}{\textbf{Non-referenced Based}}  & \multicolumn{4}{c|}{\textbf{Entropy}} \\ \hline
\textbf{} & \multicolumn{1}{c|}{\textbf{PSNR}} & \multicolumn{1}{c|}{\textbf{SSIM}} & \multicolumn{1}{c|}{\textbf{M-SSIM}} & \textbf{MSE} & \multicolumn{1}{c|}{\textbf{Brisque}} & \multicolumn{1}{c|}{\textbf{Niqe}} & \textbf{Piqe} & \multicolumn{1}{c|}{\textbf{Red}} & \multicolumn{1}{c|}{\textbf{Green}} & \multicolumn{1}{c|}{\textbf{Blue}} & \textbf{Mean} \\ \hline
\textbf{Original} & \multicolumn{1}{c|}{-} & \multicolumn{1}{c|}{-} & \multicolumn{1}{c|}{-} & - & \multicolumn{1}{c|}{28.30} & \multicolumn{1}{c|}{3.26} & 62.27 & \multicolumn{1}{c|}{7.2527} & \multicolumn{1}{c|}{7.2830} & \multicolumn{1}{c|}{6.8378} & 7.1245 \\ \hline
\textbf{Encryption~\cite{zhou2022new}} & \multicolumn{1}{c|}{-39.82} & \multicolumn{1}{c|}{0.0011} & \multicolumn{1}{c|}{0.0} & 9618.50 & \multicolumn{1}{c|}{44.03} & \multicolumn{1}{c|}{24.08} & 80.27& \multicolumn{1}{c|}{7.7302} & \multicolumn{1}{c|}{7.7656} & \multicolumn{1}{c|}{7.4314} & 7.6424 \\ \hline
\textbf{Down-sampling~\cite{kim2019privacy}} & \multicolumn{1}{c|}{-32.15} & \multicolumn{1}{c|}{0.1408} & \multicolumn{1}{c|}{0.4302} & 1740.40 & \multicolumn{1}{c|}{45.61} & \multicolumn{1}{c|}{9.57} & 66.76& \multicolumn{1}{c|}{7.7692} & \multicolumn{1}{c|}{7.8281} & \multicolumn{1}{c|}{7.1779} & 7.5917 \\ \hline
\textbf{Pixelation~\cite{fan2018image}} & \multicolumn{1}{c|}{-30.22} & \multicolumn{1}{c|}{0.1862} & \multicolumn{1}{c|}{0.6011} & 1106.40 & \multicolumn{1}{c|}{43.46} & \multicolumn{1}{c|}{28.80} & 65.23& \multicolumn{1}{c|}{7.1892} & \multicolumn{1}{c|}{7.3294} & \multicolumn{1}{c|}{6.8794} & 7.1327 \\ \hline
\textbf{Noise obfuscation~\cite{rajput2020privacy}} & \multicolumn{1}{c|}{-33.19} & \multicolumn{1}{c|}{0.1282} & \multicolumn{1}{c|}{0.3987} & 2095.90 & \multicolumn{1}{c|}{52.09} & \multicolumn{1}{c|}{11.34} & 77.02& \multicolumn{1}{c|}{7.4732} & \multicolumn{1}{c|}{7.4735} & \multicolumn{1}{c|}{7.4314} & 7.6424 \\ \hline
\textbf{Scrambling~\cite{jeevitha2021novel}} & \multicolumn{1}{c|}{-40.21} & \multicolumn{1}{c|}{0.0026} & \multicolumn{1}{c|}{0.0045} & 11197.0 & \multicolumn{1}{c|}{44.56} & \multicolumn{1}{c|}{31.97} & 41.24& \multicolumn{1}{c|}{7.8224} & \multicolumn{1}{c|}{7.7973} & \multicolumn{1}{c|}{7.2116} & 7.6105 \\ \hline
\textbf{{\em Bit-ViP} (block $5 \times 5$)} & \multicolumn{1}{c|}{-40.82} & \multicolumn{1}{c|}{0.0182} & \multicolumn{1}{c|}{0.0669} & 14187.0 & \multicolumn{1}{c|}{43.48} & \multicolumn{1}{c|}{22.08} & 74.98& \multicolumn{1}{c|}{7.6645} & \multicolumn{1}{c|}{7.0276} & \multicolumn{1}{c|}{7.8066} & 7.4996 \\ \hline
\textbf{{\em Bit-ViP} (block $10 \times 10$)} & \multicolumn{1}{c|}{-40.68} & \multicolumn{1}{c|}{0.0034} & \multicolumn{1}{c|}{0.0252} & 13056.0 & \multicolumn{1}{c|}{43.59} & \multicolumn{1}{c|}{29.29} & 76.44& \multicolumn{1}{c|}{7.2325} & \multicolumn{1}{c|}{7.2822} & \multicolumn{1}{c|}{7.9993} & 7.5047 \\ \hline
\textbf{{\em Bit-ViP} (block $20 \times 20$)} & \multicolumn{1}{c|}{-40.90} & \multicolumn{1}{c|}{0.0043} & \multicolumn{1}{c|}{0.0074} & 12837.0 & \multicolumn{1}{c|}{43.46} & \multicolumn{1}{c|}{28.18} & 76.29& \multicolumn{1}{c|}{7.5473} & \multicolumn{1}{c|}{7.4598} & \multicolumn{1}{c|}{7.6690} & 7.5587 \\ \hline
\textbf{{\em Bit-ViP} (block $25 \times 25$)} & \multicolumn{1}{c|}{-40.51} & \multicolumn{1}{c|}{0.0038} & \multicolumn{1}{c|}{0.0333} & 11508.0 & \multicolumn{1}{c|}{44.94} & \multicolumn{1}{c|}{45.11} & 78.83& \multicolumn{1}{c|}{7.9244} & \multicolumn{1}{c|}{7.6297} & \multicolumn{1}{c|}{7.3957} & 7.6499 \\ \hline
\textbf{{\em Bit-ViP} (block $40 \times 40$)} & \multicolumn{1}{c|}{-41.24} & \multicolumn{1}{c|}{0.0058} & \multicolumn{1}{c|}{0.0397} & 13318.0 & \multicolumn{1}{c|}{57.45} & \multicolumn{1}{c|}{61.29} & 80.50& \multicolumn{1}{c|}{\textbf{7.4704}} & \multicolumn{1}{c|}{\textbf{7.5294}} & \multicolumn{1}{c|}{\textbf{7.9943}} & \textbf{7.6647} \\ \hline
\end{tabular}
\label{structural_analysis}
\vspace{-0.1in}
\end{table*}

\vspace{-0.1in}
\subsection{Image Structural Analysis}\label{mean_square_err}
We conduct image structural analysis using \textit{reference-} and \ textit {non-reference}- based evaluation metrics~\cite{wang2006modern} to quantify the structural distortion (i.e., randomness) in the obfuscated image against the original input image.
\subsubsection{\textbf{Reference based analysis}}\label{reference_metrics}
In this analysis, the perceptual quality of $I_{obs}$ is compared against the original image $I$, which is commonly measured by the following four metrics: {\em{structural similarity index measure}} (SSIM), {\em{peak signal-to-noise ratio}} (PSNR), {\em{multi-scale SSIM}} (M-SSIM), and {\em{mean square error}} (MSE), and the obtained results are reported in Table~\ref{structural_analysis}. The lower the PSNR, SSIM, and M-SSIM scores, the better the scheme; the converse holds for the MSE score. While the proposed scheme (regardless of block size) outperforms schemes such as down-sampling \cite{kim2019privacy}, pixelation \cite{fan2018image}, and noise-obfuscation \cite{rajput2020privacy}, it reveals more visual information than scrambling \cite{jeevitha2021novel} and encryption \cite{zhou2022new}. However, the image obtained after scrambling or encryption does not contain learnable information that a machine-learning model can leverage, thereby making {\em Bit-ViP} a favorable option.      

\subsubsection{\textbf{Non-reference based analysis}} \label{non_reference_metrics}
Non-reference-based metrics~\cite{zhai2020perceptual} are powerful because they can evaluate the quality of $I_{obs}$ without its plain form by using images of natural scenes that exhibit similar distortions. Table \ref{structural_analysis} reports the results using three commonly employed metrics: \textit{blind image spatial quality evaluator} (BRISQUE), \textit{natural image quality evaluator} (NIQE), and \textit{perception-based image quality evaluator} (PIQE). The higher the score, the better the scheme. We observe that the $40 \times 40$ block size is the best at concealing information, yielding substantial gains across all three metrics compared to existing schemes.

\subsection{Security Analysis}\label{security_ana}
This section presents the security analysis of our proposed image obfuscation scheme and information-extracting cryptographic attacks used by adversaries to extract information from an obfuscated image and conceal an individual's identity.

\subsubsection{\textbf{Image reconstruction attack}}\label{image_reconstruction}
It involves reverse training the obfuscation scheme to recover original images from obfuscated ones. This process exploits weaknesses in the obfuscation method to reconstruct the input data, compromising privacy and security. For experiments, we performed inverse training~\cite{geiping2020inverting} of learning-based (or data-driven based) obfuscation function $\Phi^ {}(\cdot )$ presented in~\cite{dave2022spact} to obtain $\Phi^{-1} (\cdot)$. The obtained visuals are presented in Fig.~\ref{reconstruction_attack}. For the original video frames in Fig.\ref{reconstruction_attack}(a), (b) indicates the obfuscated frames using $\Phi$, and Fig.~\ref{reconstruction_attack}(c) shows the reconstructed frames using $\Phi^{-1}$, revealing all sensitive visual information. In contrast, our scheme produces a more sophisticated obfuscated image (Fig.~\ref{reconstruction_attack}(d)) and is robust against reconstruction and de-identification attacks, as shown in Fig.~\ref{reconstruction_attack}(e).

\begin{figure}[!h]
		\begin{minipage}[b]{\columnwidth}
			\centering
			\includegraphics[width=1\linewidth,height=4cm] {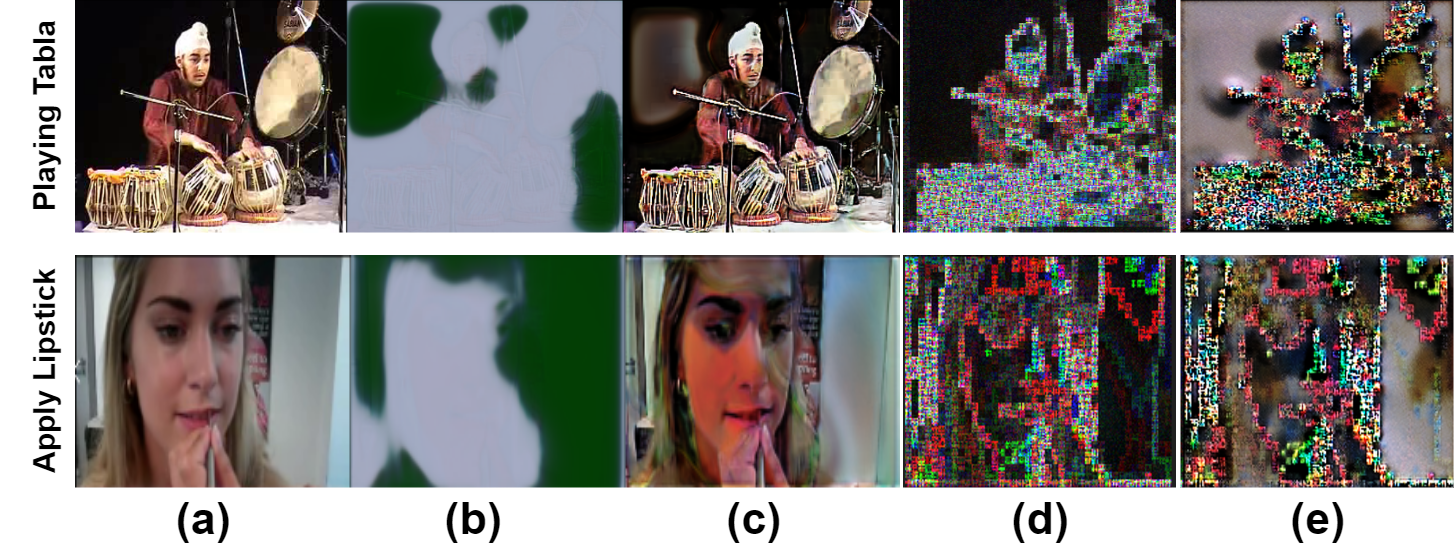}
			\caption{(a) Original image from UCF101, (b) Obfuscated by~\cite{dave2022spact}, and (c) Reconstructed image of (b), (d) Our obfuscated image, (e) Reconstructed image of (d).}
			\label{reconstruction_attack} 
		\end{minipage}
\end{figure}

\subsubsection{\textbf{Pixel-frequency based attack}}\label{histogram_ana}
In this attack, an adversary computes the frequency of each pixel intensity of the obfuscated image ($I_{obs}$) and attempts to extract meaningful information from the original image ($I$)~\cite{xian2021fractal}. Though it may seem like a trivial attack, the adversary may learn quite a bit about the original image if an underlying obfuscation scheme does not disrupt these frequencies. Ideally, the frequencies of $I_{obs}$'s pixels must be uniform and unrelated to $I$. Fig.~\ref{histogram_ana_figure} depicts channel-wise pixel-frequencies for (a) \textit{hair blowing plain image} (shown in Fig.~\ref{proposed_obfuscation_schemes_block_comparison}), (b) encryption~\cite{zhou2022new}, (c) noise obfuscation~\cite{rajput2020privacy}, and (d) proposed scheme using $40 \times 40$ block. Encryption offers more uniform (almost ideal) pixel frequencies than the proposed one. However, the DNN model cannot be trained on encrypted images, resulting in poor accuracy (see Table~\ref{accuracy_comparison}). In contrast, our scheme has approximately uniform frequencies after obfuscation while securing a usable accuracy. 

\begin{figure}[!h]
		\begin{minipage}{\columnwidth}
			\centering
			\includegraphics[width=1.0\linewidth]{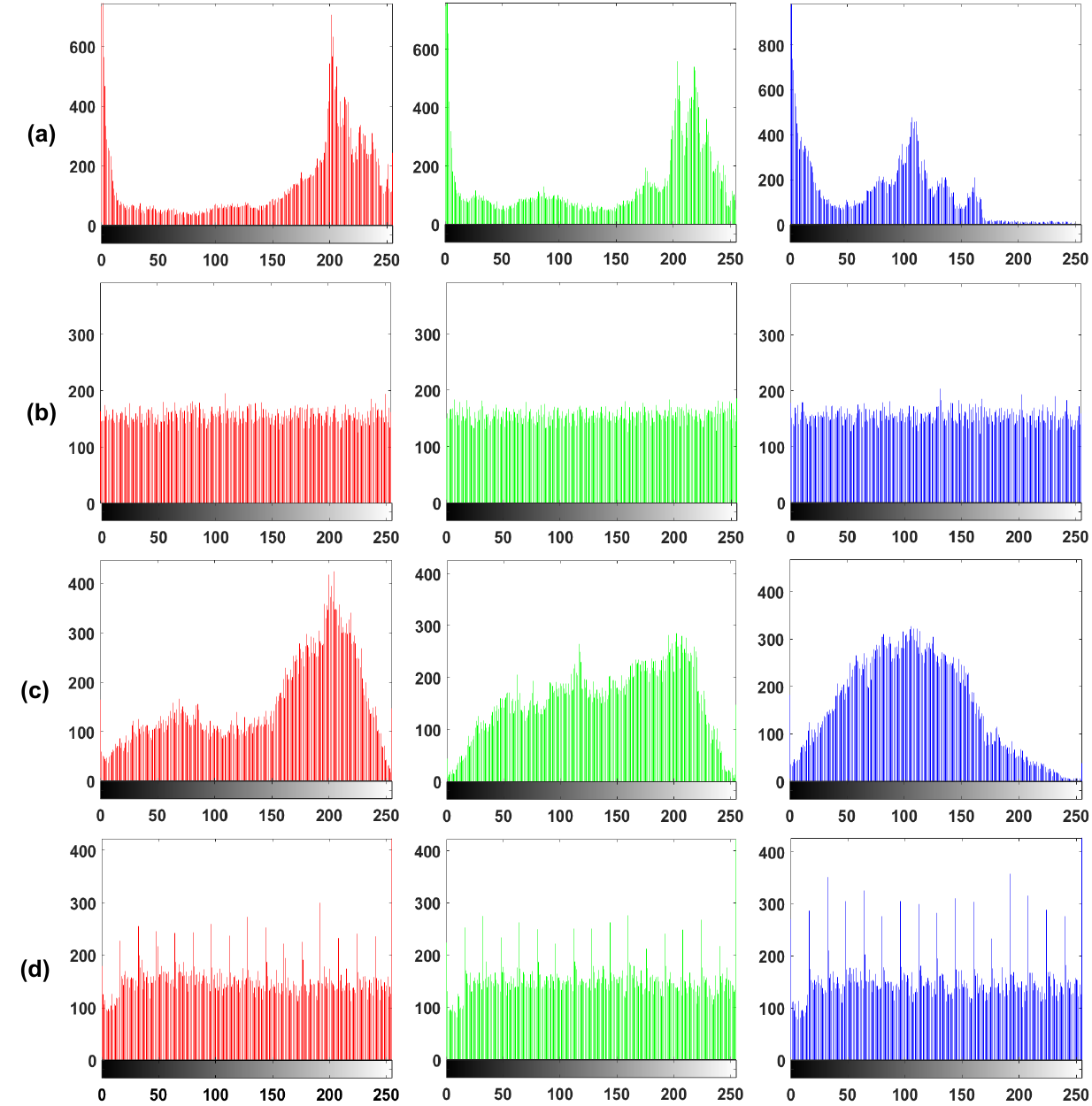}	
			\caption{Channel-wise pixel frequencies for (a) original image, (b) encryption scheme, (c) noise obfuscation scheme, and (d) proposed scheme with block size $40 \times 40$. The x-axis and y-axis indicate the pixel intensity values in $\left[0, 255\right]$ and their frequencies, respectively. } 
		\label{histogram_ana_figure} 
		\end{minipage}
\end{figure}

\begin{figure*}[!h]
			\centering
			\includegraphics[width=0.85\linewidth,height=6.5cm]{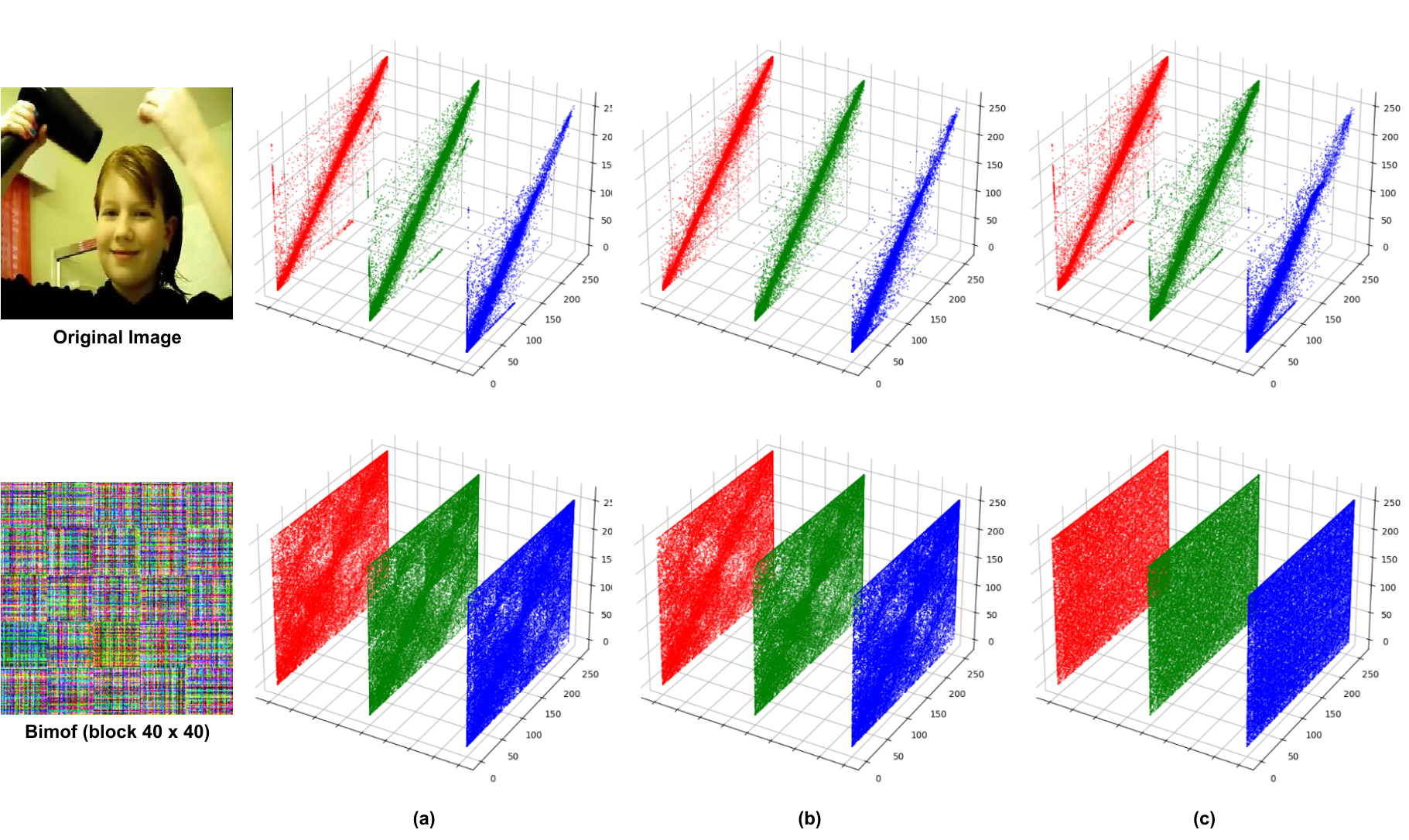}
			\caption{Qualitative analysis of the correlation of the original and the {\em Bit-ViP} obfuscated images for red, green, and blue color channels in (a) horizontal, (b) vertical, and (c) diagonal directions.}
		\label{corr_result}
\end{figure*}

\subsubsection{\textbf{Information entropy}}\label{information_entropy}
Entropy quantifies the amount of randomness present in the data. Higher entropy is desirable for concealing private visual information of an obfuscation image. For an $N$-bit image $I_{obs}$ with $T$ distinct pixel intensities, the entropy lies in range $[0, N]$, and is computed as $-\sum_{t=0}^{T-1}p(t)\hspace{0.05cm}log_2 (p(t))$. We compute entropy over 10 randomly chosen images from both datasets, including one shown in Fig.~\ref{proposed_obfuscation_schemes_block_comparison}, and report the average channel-wise results in Table~\ref{structural_analysis}. It clearly shows that the effectiveness of our scheme is comparable to that of all existing schemes, particularly with large block sizes. With block $40 \times 40$, the mean entropy is $7.67$, close to the maximum possible value (8).  

\subsubsection{\textbf{Differential analysis}}
Through differential analysis, we examine changes in two distinct obfuscated forms of a plain image $I \in \mathbb{R}^{M\times N}$ under small perturbations. Specifically, we perturb $I$ with a pixel intensity value and obtain obfuscated forms before and after perturbation, say $I_{obs_1}$ and $I_{obs_2}$, and then compute their difference percentage ($\mathcal{D}\mathcal{P}$) as 
\begin{align}\nonumber
    \mathcal{D}\mathcal{P} =&  {\frac{\sum_{m=1}^M\sum_{n=1}^N \mathcal{D}(m, n)}{M\times N}} \times 100,\\ \nonumber
\text{where} \quad    \mathcal{D}(m, n) &= \begin{cases}1, & \textbf{if } I_{obs_2}(m, n) \neq I_{obs_1}(m, n) \\0, & \textbf{if } I_{obs_2}(m, n) = I_{obs_1}(m, n).\end{cases}
\end{align}
Table~\ref{differential_analysis} presents the channel-wise scores averaged over ten randomly chosen images from both datasets.

\begin{table}[!h]
    \centering
	\caption{$\mathcal{D}\mathcal{P}$ for different obfuscation schemes. }
{
\begin{tabular}{|l|c|c|c|c|}
\hline
\multicolumn{1}{|c|}{} & \textbf{Red} & \textbf{Green} & \textbf{Blue} & \textbf{Mean} \\ \hline
\textbf{Encryption \cite{zhou2022new}} & 99.63\% & 99.66\% & 99.62\% & 99.64\% \\ \hline
\textbf{Down-sampling \cite{kim2019privacy}} & 0.0\% & 0.0\% & 0.0\% & 0.0\% \\ \hline
\textbf{Pixelation \cite{fan2018image}} & 67.73\% & 60.38\% & 68.14\% & 65.42\% \\ \hline
\textbf{Noise obfuscation \cite{rajput2020privacy}} & 93.98\% & 92.94\% & 92.54\% & 93.15\% \\ \hline
\textbf{Scrambling \cite{jeevitha2021novel}} & 97.31\% & 97.30\% & 96.69\% & 97.10\% \\ \hline
\textbf{{\em Bit-ViP} (block $5 \times 5$)} & 88.97\% & 87.82\% & 85.83\% & 87.54\% \\ \hline
\textbf{{\em Bit-ViP} (block $10 \times 10$)} & 89.49\% & 89.76\% & 91.27\% & 90.17\% \\ \hline
\textbf{{\em Bit-ViP} (block $20 \times 20$)}& 92.44\% & 90.79\% & 91.55\% & 91.59\% \\ \hline
\textbf{{\em Bit-ViP} (block $25 \times 25$)}& 93.69\% & 92.25\% & 93.62\% & 93.19\% \\ \hline
\textbf{{\em Bit-ViP} (block $40 \times 40$)} & 95.01\% & 94.56\% & 91.92\% & 93.83\% \\ \hline
\end{tabular}
\label{differential_analysis}}
\end{table}

\begin{table*}[!h]
\vspace{-0.1in}
\setlength{\tabcolsep}{6.75pt}
    \renewcommand{\arraystretch}{1.2}
\centering
\caption{Quantitative comparison of the correlation. R, G, B, and M denote red, green, and blue color channels and the mean of R, G, and B, respectively.}
\resizebox{\textwidth}{!}{
\begin{tabular}{|l|cccc|cccc|cccc|}
\hline
\multicolumn{1}{|c|}{\multirow{2}{*}{{\textbf{Direction}}}} & \multicolumn{4}{c|}{\textbf{Horizontal}} & \multicolumn{4}{c|}{\textbf{Vertical}} & \multicolumn{4}{c|}{\textbf{Diagonal}} \\ \cline{2-13} 
\multicolumn{1}{|c|}{} & \multicolumn{1}{c|}{\textbf{R}} & \multicolumn{1}{c|}{\textbf{G}} & \multicolumn{1}{c|}{\textbf{B}} & \textbf{M} & \multicolumn{1}{c|}{\textbf{R}} & \multicolumn{1}{c|}{\textbf{G}} & \multicolumn{1}{c|}{\textbf{B}} & \textbf{M} & \multicolumn{1}{c|}{\textbf{R}} & \multicolumn{1}{c|}{\textbf{G}} & \multicolumn{1}{c|}{\textbf{B}} & \textbf{M} \\ \hline
\textbf{Original} & \multicolumn{1}{c|}{0.986} & \multicolumn{1}{c|}{0.985} & \multicolumn{1}{c|}{0.976} & 0.982 & \multicolumn{1}{c|}{0.988} & \multicolumn{1}{c|}{0.987} & \multicolumn{1}{c|}{0.978} & 0.984 & \multicolumn{1}{c|}{0.979} & \multicolumn{1}{c|}{0.979} & \multicolumn{1}{c|}{0.965} & 0.947 \\ \hline
\textbf{Encryption~\cite{zhou2022new}} & \multicolumn{1}{c|}{-0.005} & \multicolumn{1}{c|}{-0.035} & \multicolumn{1}{c|}{-0.012} & -0.020 & \multicolumn{1}{c|}{-0.02} & \multicolumn{1}{c|}{0.041} & \multicolumn{1}{c|}{-0.010} & 0.003 & \multicolumn{1}{c|}{0.010} & \multicolumn{1}{c|}{0.026} & \multicolumn{1}{c|}{0.015} & 0.017 \\ \hline
\textbf{Down-sampling~\cite{kim2019privacy}} & \multicolumn{1}{c|}{0.999} & \multicolumn{1}{c|}{0.999} & \multicolumn{1}{c|}{0.999} & 0.999 & \multicolumn{1}{c|}{0.999} & \multicolumn{1}{c|}{0.999} & \multicolumn{1}{c|}{0.999} & 0.999 & \multicolumn{1}{c|}{0.998} & \multicolumn{1}{c|}{0.998} & \multicolumn{1}{c|}{0.998} & 0.998 \\ \hline
\textbf{Pixelation~\cite{fan2018image}} & \multicolumn{1}{c|}{0.989} & \multicolumn{1}{c|}{0.988} & \multicolumn{1}{c|}{0.986} & 0.988 & \multicolumn{1}{c|}{0.990} & \multicolumn{1}{c|}{0.989} & \multicolumn{1}{c|}{0.987} & 0.989 & \multicolumn{1}{c|}{0.964} & \multicolumn{1}{c|}{0.963} & \multicolumn{1}{c|}{0.962} & 0.963 \\ \hline
\textbf{Noise obfuscation~\cite{rajput2020privacy}} & \multicolumn{1}{c|}{0.998} & \multicolumn{1}{c|}{0.998} & \multicolumn{1}{c|}{0.997} & 0.998 & \multicolumn{1}{c|}{0.998} & \multicolumn{1}{c|}{0.998} & \multicolumn{1}{c|}{0.997} & 0.998 & \multicolumn{1}{c|}{0.997} & \multicolumn{1}{c|}{0.997} & \multicolumn{1}{c|}{0.995} & 0.996 \\ \hline
\textbf{Scrambling~\cite{jeevitha2021novel}} & \multicolumn{1}{c|}{0.693} & \multicolumn{1}{c|}{0.687} & \multicolumn{1}{c|}{0.671} & 0.684 & \multicolumn{1}{c|}{0.694} & \multicolumn{1}{c|}{0.668} & \multicolumn{1}{c|}{0.679} & 0.680 & \multicolumn{1}{c|}{0.476} & \multicolumn{1}{c|}{0.476} & \multicolumn{1}{c|}{0.467} & 0.473 \\ \hline
\textbf{\begin{tabular}[c]{@{}l@{}}{\em Bit-ViP} (block $5 \times 5$)\end{tabular}} & \multicolumn{1}{c|}{0.620} & \multicolumn{1}{c|}{0.589} & \multicolumn{1}{c|}{0.538} & 0.583 & \multicolumn{1}{c|}{0.633} & \multicolumn{1}{c|}{0.611} & \multicolumn{1}{c|}{0.535} & 0.593 & \multicolumn{1}{c|}{0.468} & \multicolumn{1}{c|}{0.397} & \multicolumn{1}{c|}{0.412} & 0.426 \\ \hline
\textbf{\begin{tabular}[c]{@{}l@{}}{\em Bit-ViP} (block $10 \times 10$)\end{tabular}} & \multicolumn{1}{c|}{0.686} & \multicolumn{1}{c|}{0.633} & \multicolumn{1}{c|}{0.665} & 0.661 & \multicolumn{1}{c|}{0.667} & \multicolumn{1}{c|}{0.662} & \multicolumn{1}{c|}{0.666} & 0.665 & \multicolumn{1}{c|}{0.506} & \multicolumn{1}{c|}{0.496} & \multicolumn{1}{c|}{0.494} & 0.499 \\ \hline
\textbf{\begin{tabular}[c]{@{}l@{}}{\em Bit-ViP} (block $20 \times 20$)\end{tabular}} & \multicolumn{1}{c|}{0.552} & \multicolumn{1}{c|}{0.553} & \multicolumn{1}{c|}{0.546} & 0.550 & \multicolumn{1}{c|}{0.533} & \multicolumn{1}{c|}{0.536} & \multicolumn{1}{c|}{0.563} & 0.544 & \multicolumn{1}{c|}{0.337} & \multicolumn{1}{c|}{0.324} & \multicolumn{1}{c|}{0.362} & 0.341 \\ \hline
\textbf{\begin{tabular}[c]{@{}l@{}}{\em Bit-ViP} (block $25 \times 25$)\end{tabular}} & \multicolumn{1}{c|}{0.285} & \multicolumn{1}{c|}{0.457} & \multicolumn{1}{c|}{0.507} & 0.417 & \multicolumn{1}{c|}{0.230} & \multicolumn{1}{c|}{0.470} & \multicolumn{1}{c|}{0.513} & 0.404 & \multicolumn{1}{c|}{0.246} & \multicolumn{1}{c|}{0.204} & \multicolumn{1}{c|}{0.255} & 0.235 \\ \hline
\textbf{\begin{tabular}[c]{@{}l@{}}{\em Bit-ViP} (block $40 \times 40$)\end{tabular}} & \multicolumn{1}{c|}{0.376} & \multicolumn{1}{c|}{0.329} & \multicolumn{1}{c|}{0.312} & 0.339 & \multicolumn{1}{c|}{0.337} & \multicolumn{1}{c|}{0.289} & \multicolumn{1}{c|}{0.297} & 0.308 & \multicolumn{1}{c|}{0.013} & \multicolumn{1}{c|}{-0.010} & \multicolumn{1}{c|}{0.032} & 0.012 \\ \hline
\end{tabular}
\label{corr_coff}
}
\vspace{-0.15in}
\end{table*}

The obtained mean (over channels) $\mathcal{D}\mathcal{P}$ lies in the range 87\%-94\% (higher is better), signifying sufficient variations in the pixel intensities of two obfuscated images $I_{obs_1}$ and $I_{obs_2}$. Our scheme achieved better scores with a $40 \times 40$ size. The interesting point is that $\mathcal{D}\mathcal{P}$ is zero for down-sampling~\cite{kim2019privacy} because it applies the same kernel on each image, whereas the pixelation~\cite{fan2018image} uses a fixed range Laplacian noise and thereby secures $\approx 65\%$. The scrambling~\cite{jeevitha2021novel} randomly permutes pixel locations every time, thus achieving competitive scores. Further, encryption~\cite{zhou2022new} achieved better scores than our scheme, but without supporting model learning on the data; therefore, encryption alone is not acceptable in machine learning applications. 

\subsubsection{\textbf{Pixels-correlation attack}} \label{cor_coef}
It (denoted as $corr$) signifies the interdependence among the pixel intensity values across horizontal, vertical, and diagonal orientations within an image~\cite{chen2004symmetric}. Elevated pixel correlation suggests a stronger presence of significant content and an increased susceptibility to leakage. To mitigate this risk, the correlation of an obscured image should be inversely related to that of the original image, ideally nearing zero. In this attack, the adversary can access the obfuscated images and use correlation between neighboring pixels to extract meaningful regions. Mathematically, $corr$ is evaluated over random $ K $ neighboring pixels say $ \left\{(x_1, y_1), (x_2, y_2),...,(x_K, y_K)\right\} $ in an image as
\begin{equation}
corr(x,y)=\frac{cov(x,y)}{S_tD(x)\times S_tD(y)}
\end{equation}
where $S_tD$ and $ cov(x,y) $ represent the standard deviation and covariance respectively between pixel intensities $x = \left\{x_1, x_2,...,x_K\right\} $ and $ y= \left\{y_1, y_2,...,y_K\right\} $. $corr$ lies in $ \left[\text{-}1, 1\right] $, with ``-$1"$ and ``$ 1 $" indicating perfect inverse and direct proportionality, whereas $corr$ close to $ ``0" $ depicts no relation between $x$ and $y$.

We compute the $corr$ of the gray-scale {\em{hair blowing}} image as shown in Fig.~\ref{corr_result} and {\em Bit-ViP} images with a block size of $40\times 40$. Qualitatively, the pixels of the original and obfuscated images would be projected diagonally and scattered, respectively, as shown in Fig.~\ref{corr_result}. For experiments, we set $K = 2000$, and the average correlation values for $100$ random images in the horizontal, vertical, and diagonal directions are reported in Table~\ref{corr_coff}. The correlation of original images lies in $\left[0.965, 0.988 \right]$ whereas {\em Bit-ViP} has range $\left[-0.01, 0.686\right]$, depicting that {\em Bit-ViP} significantly preserves V-PII. It is observed that {\em Bit-ViP} outperforms existing schemes~\cite{kim2019privacy,fan2018image,rajput2020privacy,jeevitha2021novel}.

\subsection{Computation Time} \label{perform_ana}
The time taken by an obfuscation scheme at the local device is essential for cloud services, which we discuss in this section. We experimented with $10$ images randomly selected from both datasets to compute time with varying image dimensions and block sizes, and the results are presented in Table~\ref{obfuscation_time}. Each value is the average time to obfuscate the given $10$ images for the indicated dimension and block size. For instance, if a $32\times 32$-dimensional image is broken into blocks of size $5\times 5$ (with padding), then a total of $49$ blocks are produced. Our scheme takes $0.011$ seconds per block, so it takes $0.539$ seconds $\left(0.539=0.011 \times 49\right)$ for all the blocks.    

\begin{table}[!h]
	\centering
	\caption{Obfuscation time for varying image dimensions and block sizes.} 
    \begin{tabular}{|l|c|c|c|c|c|} \hline
        \multicolumn{1}{|c|}{\multirow{2}{*}{\textbf{\begin{tabular}[c]{@{}c@{}}Image \\ size\end{tabular}}}} & \multicolumn{5}{c|}{\textbf{{\em Bit-ViP} (time in seconds)}} \\ \cline{2-6} \multicolumn{1}{|c|}{} & \textbf{$5 \times 5$} & \textbf{$10 \times 10$} & \textbf{$20 \times 20$} & \textbf{$25 \times 25$} & \textbf{$40 \times 40$} \\ \hline
        \textbf{$32 \times 32$} & 2.37 & 0.45 & 0.18 & 0.16 & 0.11 \\ \hline
        \textbf{$128 \times 128$} & 20.73 & 4.14 & 1.09 & 0.84 & 0.51 \\ \hline
        \textbf{$256 \times 256$} & 81.41 & 20.83 & 5.71 & 3.69 & 1.13 \\ \hline
        \textbf{$512 \times 512$} & 358.13 & 94.37 & 27.97 & 18.56 & 10.49 \\ \hline
    \end{tabular}
    \label{obfuscation_time}
\end{table}

The time required for pre-processing, such as splitting the image into color channels followed by generation of $5 \times 5$ blocks, post-processing operations, and concatenating the obfuscated blocks of each color channel of a $32 \times 32$-dimensional image, is $1.831$ seconds, making the total obfuscation time $2.37 = 1.831+0.539 $ seconds. Note that the same image requires $0.11$ seconds if the block size is $40 \times 40$. \vspace{-0.1cm}
\section{Conclusion}\label{conclusion}
We introduced {\em Bit-ViP}, a novel scheme based on bit-planes to safeguard individuals' private visual information within images uploaded to cloud storage while maintaining data security. By integrating Lorenz's chaotic noise across bit planes and performing subsequent QR decomposition, our scheme achieves resilience against adversarial attacks, such as reconstruction and model inversion, while retaining sufficient information in the obfuscated data for effective model training. Theoretical analysis confirmed that the obfuscated images, essentially the dataset, adhere to the DP exponential mechanism at the block level, indicating the non-invertible nature of the {\em Bit-ViP} scheme. Extensive experimentation on real-world human activity recognition datasets (UCF101 and HMDB51) provides a qualitative and quantitative evaluation of the proposed scheme, comprehensively assessing security, usability, and image structure. Results demonstrated that {\em Bit-ViP} effectively mitigates V-PII exposure by disrupting pixel correlations, altering pixel frequencies, and introducing random noise, while significantly surpassing previous schemes in accuracy. Furthermore, we explored a security-usability trade-off by varying block sizes, empowering users to select the most suitable size for their application requirements.

We plan to exploit the correlation between the original image and its obfuscated version to gain recognition accuracy. This work motivates further research to incorporate users' choices in selecting a region of interest, such as faces or secret objects to be obfuscated, thereby requiring fine-tuning of randomness based on the importance level of the visual information. 

\vspace{-0.5cm}
\bibliographystyle{IEEEtran}
\bibliography{mybib_1}
\vspace{-2cm}
\begin{IEEEbiography}[{\includegraphics[width=1in,height=1.15in]{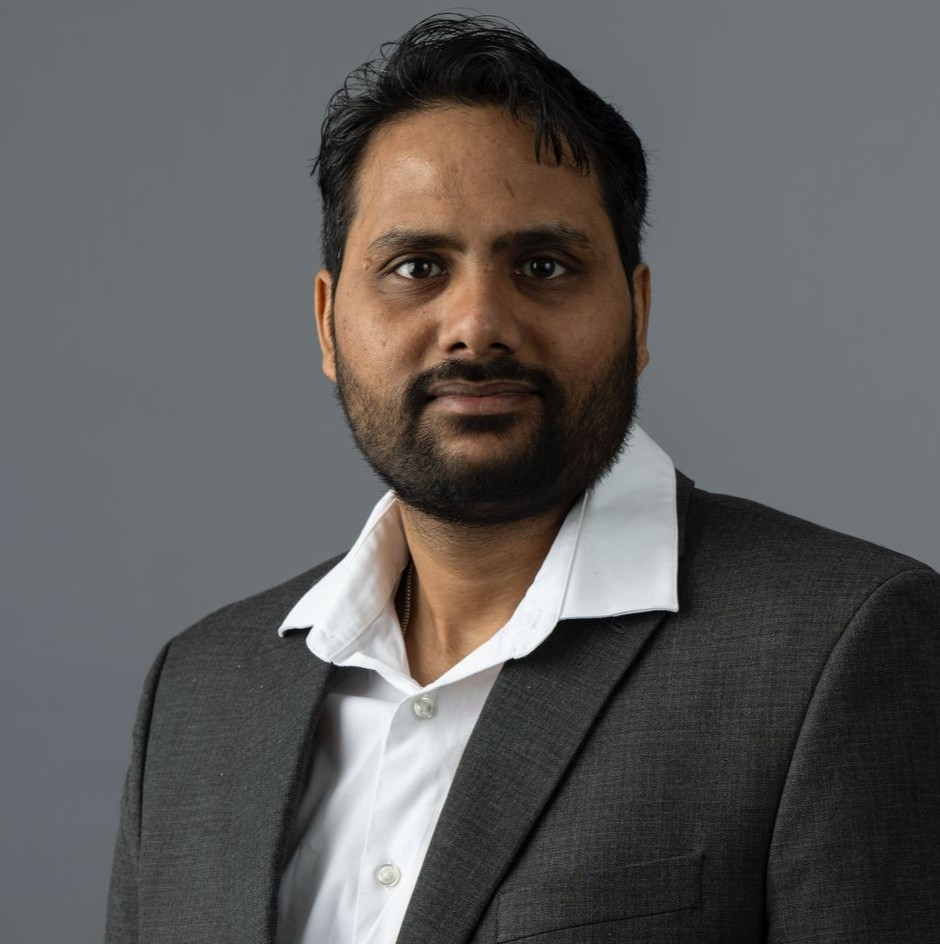}}]{Vishesh Kumar Tanwar} is a Senior Research Associate in the Department of Computer Science at Missouri S\&T, USA. He received his Ph.D. in applied mathematics from the Indian Institute of Technology Roorkee, India. His research interests include privacy-preserving multimedia processing, split learning, and cloud computing. 
\end{IEEEbiography}

\vspace{-2.5cm}
\begin{IEEEbiography}[{\includegraphics[width=1in,height=1.25in, clip]{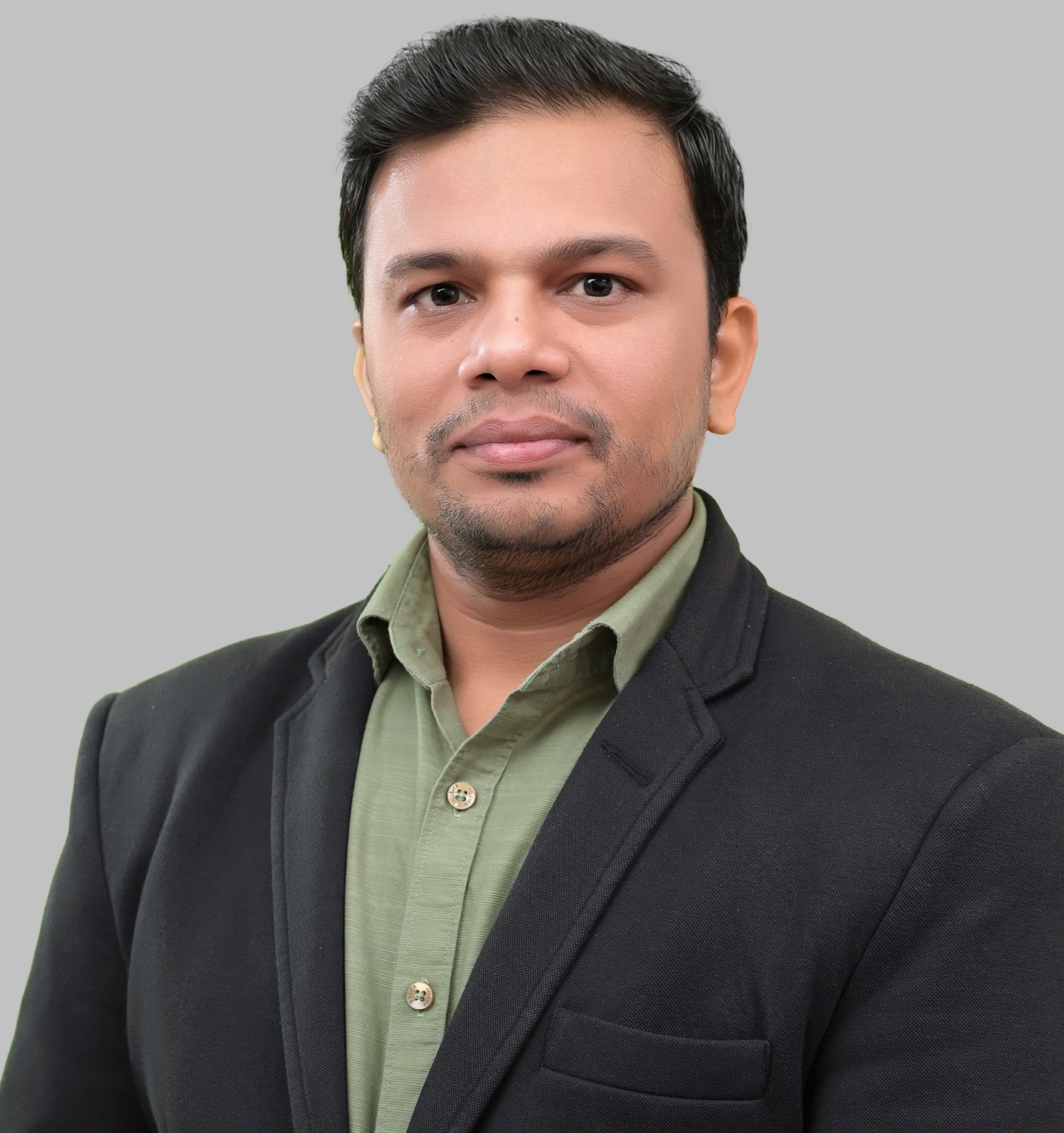}}]{Ashish Gupta} is an Assistant Professor in Computer Science at BITS Pilani Dubai Campus, UAE. He worked as a postdoctoral fellow in the Department of Computer Science at Missouri S\&T, USA, from 2021 to 2023. He received a Ph.D. in Computer Science and Engineering from the Indian Institute of Technology (BHU), Varanasi, India. His research interests include sensor data analytics, federated learning, and applied machine learning. 
\end{IEEEbiography}
\vspace{-2.5cm}
\begin{IEEEbiography}[{\includegraphics[width=1in,height=1.125in, clip]{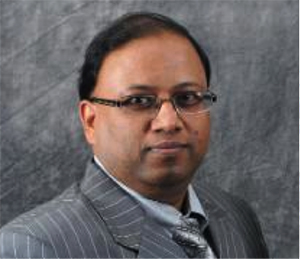}}]{Sanjay Madria} is a Curators' Distinguished Professor in the Department of Computer Science at Missouri S\&T, USA. He has published 300+ journal and conference papers on mobile and sensor computing, big data and cloud computing, and cybersecurity. He has been awarded the Japanese Society for the Promotion of Science Invitational Visiting Scientist Fellowship, the American Society for Engineering Education Fellowship, and an ACM Distinguished Scientist. He is an IEEE Senior Member and an IEEE Golden Core Awardee. 
\end{IEEEbiography}
\vspace{-2.25cm}
\begin{IEEEbiography}[{\includegraphics[width=1in,height=1.25in, clip]{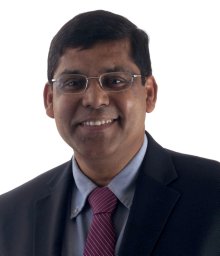}}]{Sajal K. Das} is a Curators' Distinguished Professor of computer science and Daniel St. Clair Endowed Chair at Missouri S\&T, USA. His research interests include cyber-physical systems, IoT, smart environments (including smart agriculture), wireless sensor networks, pervasive and mobile computing, and cybersecurity. He is the editor-in-chief of Elsevier's Pervasive and Mobile Computing journal and an associate editor of the IEEE Transactions on Mobile Computing and IEEE Transactions on Dependable and Secure Computing.
\end{IEEEbiography}
\end{document}